%% file: neurips_2023.tex
\documentclass{article}

    \PassOptionsToPackage{numbers, compress}{natbib}
\usepackage[preprint]{neurips_2023}

\usepackage[utf8]{inputenc} % allow utf-8 input
\usepackage[T1]{fontenc}    % use 8-bit T1 fonts
\usepackage{hyperref}       % hyperlinks
\usepackage{url}            % simple URL typesetting
\usepackage{booktabs}       % professional-quality tables
\usepackage{amsfonts,amssymb,amsthm}       % blackboard math symbols
\usepackage{nicefrac}       % compact symbols for 1/2, etc.
\usepackage{microtype}      % microtypography
\usepackage{xcolor}         % colors
\usepackage{tikz}           % creating plots and diagrams
\usepackage{pgfplots}
\usepackage{pgfkeys}
\usepackage{algorithmic}
\usepackage{algorithm}
\usepackage{mathtools}
\usepackage{subcaption}
\usepackage{soul}

\definecolor{mytealuno}{RGB}{30,150,250}
\definecolor{mytealdos}{RGB}{20,100,150}
\definecolor{mytealtres}{RGB}{10,50,150}
\definecolor{myred}{RGB}{204,51,17}
\definecolor{mygreen}{RGB}{0,0,0}
\pgfplotsset{compat=1.10, 
cycle list={%
{draw=myteal,mark=o},
{draw=myred, mark=o}}}

\usepgfplotslibrary{fillbetween}

\DeclareMathAlphabet{\altmathcal}{OMS}{cmsy}{m}{n}
\DeclareMathAlphabet{\altmathcal}{OMS}{cmsy}{m}{n}
\DeclarePairedDelimiter\ceil{\lceil}{\rceil}
\DeclarePairedDelimiter{\norm}{\lVert}{\rVert}
\newtheorem{theorem}{Theorem}
\newtheorem{claim}{Claim}
\newtheorem{lemma}{Lemma}
\newtheorem{definition}{Definition}
\newcommand{\algoname}{Fed-FairX-LinUCB}
\newcommand{\ourprivalgo}{Priv-FairX-LinUCB}

\title{Fairness and Privacy Guarantees in Federated Contextual Bandits}

\author{%
  Sambhav Solanki \\
  IIIT, Hyderabad \\
  \texttt{sambhav.solanki@research.iiit.ac.in} \\
  \AND
  Shweta Jain \\
  IIT, Ropar \\
  \texttt{shwetajain@iitrpr.ac.in} \\
  \And
  Sujit Gujar \\
  IIIT, Hyderabad \\
  \texttt{sujit.gujar@iiit.ac.in} \\
}

\begin{document}

\maketitle

\begin{abstract}
This paper considers the contextual multi-armed bandit (CMAB) problem with fairness and privacy guarantees in a federated environment. We consider merit-based exposure as the desired \emph{fair} outcome, which provides exposure to each action in proportion to the reward associated. We model the algorithm's effectiveness using fairness regret, which captures the difference between fair optimal policy and the policy output by the algorithm. Applying fair CMAB algorithm to each agent individually leads to fairness regret linear in the number of agents. We propose that collaborative -- federated learning can be more effective and provide the algorithm \algoname\ that also ensures differential privacy. The primary challenge in extending the existing privacy framework is designing the communication protocol for communicating required information across agents. A naive protocol can either lead to weaker privacy guarantees or higher regret. We design a novel communication protocol that allows for (i) Sub-linear theoretical bounds on fairness regret for \algoname\ and comparable bounds for the private counterpart, \ourprivalgo\ (relative to single-agent learning), (ii) Effective use of privacy budget in \ourprivalgo. We demonstrate the efficacy of our proposed algorithm with extensive simulations-based experiments. We show that both \algoname\ and \ourprivalgo\ achieve near-optimal fairness regret.

\end{abstract}

% %%%%%%%%%%%%%%%%%%%%%%%%%%%%%%%%%%%%%%%%%%%%%%%%%
\section{Introduction}\label{sec:intro}
% %%%%%%%%%%%%%%%%%%%%%%%%%%%%%%%%%%%%%%%%%%%%%%%%%
The \emph{bandit} problem \cite{auer} is a well-known problem encapsulating the exploration and exploitation trade-off in online learning. It has a wide array of applications, such as crowdsourcing \cite{tran2014efficient}, recommendation systems \cite{li2010contextual}, sponsored search auctions \cite{abhishek2020designing}, service procurement \cite{badanidiyuru2013bandits}, etc. This paper considers the contextual \emph{multi-armed bandit} (MAB) problems in a \emph{federated} setting.

Linear contextual bandits \cite{li2010contextual} associate dynamic contexts with each action by assuming that the reward for each action is modeled as a fixed but unknown linear combination of the context and thus aims to learn these linear weights for maximizing the reward of a single learning agent. Multiple agents can collaborate in many real-world applications such as crowdsourcing, service procurement, and recommender systems for better effective learning~\cite{he2022a,reda2022near,solanki2022differentially}. For example, in crowdsourcing, requesters (agents) of similar tasks intend to learn the qualities of a pool of workers (actions), which are context-dependent. In such examples, agents can help each other by collaborating to learn the correlation between worker attributes (action context) and task completion proficiency (rewards) faster. Such collaborative learning should be allowed without sharing sensitive data (such as specific worker selection in any given round) among the agents while allowing for effective learning, i.e., it should protect the privacy of individual agents' sensitive information. The literature model collaborative learning with privacy requirements via the paradigm of \emph{federated learning} for practical collaboration \cite{kairouz2021advances}. Recent works~\cite{solanki2022differentially,dubey2020differentially} have explored \emph{differential privacy} guarantees in \emph{federated bandits} which extend bandit problem in federated settings. %\cite{solanki2022differentially} focuses on combinatorial bandits while \cite{dubey2020differentially} focuses on contextual bandits. 

In many practical applications, actions often involve interactions with humans, e.g., workers in crowdsourcing. Here, it becomes crucial to ensure that each action receives sufficient exposure. Traditional bandit approaches exhibit a ``winner takes all'' behaviour~\cite{wang2021fairness}, which consistently favors the optimal action and deprives other actions of opportunity, leading to starvation among actions. We address this issue by considering \emph{fairness of exposure}~\cite{wang2021fairness} in multi-agent contextual bandit problems. Other fairness notions in the context of bandit problems, such as guaranteeing minimum exposure to each action \cite{patil2020achieving}, group fairness, and fair treatment \cite{NIPS2016_eb163727} depend solely on the rewards or prioritize fairness for the learning agents rather than the individual actions. %In our view, the concept of \emph{fairness of exposure} introduced in \cite{wang2021fairness} provides the most effective solution to the starvation problem. 
On the other hand, fairness of exposure ensures proportionality~\cite{Talmund,SUKSOMPONG201662} for the actions, meaning that every action would be selected proportional to its merit/reward. This is an essential indicator of individual fairness in ML algorithms and proportionality in game theoretical frameworks. The algorithm in \cite{wang2021fairness} works only for a single-agent setting. There are a few works \cite{hossain2021fair,biswas2023fairness} that provide fairness guarantees in a federated setting; however, these works consider fairness for agents instead of actions.

Fairness of exposure in bandits focuses on minimizing \emph{fairness regret}, which measures the deviation of action selection policy from the optimal policy satisfying fairness. For the first time, this paper provides fairness regret guarantees in the federated setting while ensuring privacy. One naive way to ensure fairness in federated learning is to integrate a  \emph{communication protocol} where all the agents communicate with each other in every round by sharing all the information they have about the actions with the existing fair algorithms. The communication in every round leads to minimum fairness regret, albeit it leads to maximum privacy leakage. Another extreme is not to allow any communication among agents. It leads to maximum privacy, but in the absence of collaborative learning, the regret blooms in terms of the number of agents. Therefore, developing an intelligent communication protocol that provides a regret bound that is sub-linear in the number of agents and extends to the private setting is essential.

This work designs a novel communication protocol for federated bandits while learning generalizes the techniques from FairX-LinUCB~\cite{wang2021fairness}, an algorithm designed for a single-agent setting, to a federated setting. We call our algorithm as \algoname. Our communication protocol is scalable to differentially private methods since the number of communication rounds is bounded while ensuring fairness given the bounded communication gaps. We denote the privacy-ensuring version of the proposed algorithm by \ourprivalgo. In summary, our paper solves the fair federated contextual MAB problem while ensuring differential privacy guarantees. Our contributions include: 
% \noindent\paragraph{Contributions}
\begin{enumerate}
    \item We introduce the notion of fairness for actions in federated contextual bandits.
    \item We propose a novel communication protocol and show that \algoname\ achieves sub-linear fairness regret in terms of the number of learning agents while being optimal in terms of the number of rounds up to a \(\log\) dependence term (Theorem ~\ref{theorem: fairness regret}).\footnote{It is trivially implied that fairness regret would scale linearly in \(m\) for non-collaborative learning.} 
    \item The proposed communication is extensible to privatizer routine from \cite{dubey2020differentially}. It lets us develop \ourprivalgo, which ensures differential privacy guarantees (for the agents).
    \item We theoretically show that \ourprivalgo\ achieves differential privacy guarantees while having bounded fairness regret (Theorem~\ref{theorem: private fairness regret}).
    \item We empirically show that \algoname\ and \ourprivalgo\ outperform a non-collaborative learner.
\end{enumerate}

%%%%%%%%%%%%%%%%%%%%%%%%%%%%%%%%%%%%%%%%%%%%%%%%%
\section{Related Work}
% We divide the related work into three parts,  traditional and  contextual bandits, fairness in bandits, and differential privacy.

\noindent \textbf{Federated Bandits.} Bandit problems~\cite{robbins,LAI19854,auer} with contextual nature~\cite{li2010contextual,NIPS2011_e1d5be1c} have gained significant prominence in both academia and industry. Moreover, analysing bandit problems in a federated setting~\cite{he2022a,reda2022near} has been an important exploration of cooperative learning.

\textbf{Privacy.}  Our work leverages federated learning, which does large dataset querying. We use differential privacy, introduced by \cite{dwork2006calibrating}, to provide privacy for context/reward information. Differential privacy is a rigorous mathematical notion of privacy that encapsulates the requirement that the probability of output should have minimal changes for neighbouring input datasets. \cite{privatecontinual} and \cite{dwork2010differential} introduced the notion of differential privacy under continual observation using a \textbf{tree-based algorithm}, which we leverage. This method has seen utilisation across several online learning problems \cite{tossou2016algorithms,NIPS2013_c850371f,kairouz2021advances,jain2012differentially}. \cite{shariff2018differentially} study differential privacy for the traditional contextual bandit setting but is limited to a single learning agent. Differential private federated bandits have been studied in \cite{liu2022federated} and \cite{solanki2022differentially}. However, our work is closely related to the important work of \cite{dubey2020differentially}, extending it for non-traditional bandit optimisation.

\noindent \textbf{Fairness in Bandits.} Significant progress has been made in traditional bandits, but bandits with fairness objectives have only recently gained popularity. \cite{NIPS2016_eb163727}, propose bandit fairness which is achieved by ensuring that a better arm is always chosen with at least the same likelihood as a worse arm. Several other works, including  \cite{chen2020fair,patil2020achieving}, aim to guarantee a minimum exposure for arms in the stochastic bandit problem. However, based on the discussion in Section ~\ref{sec:intro}, it remains unclear how much exposure would be enough. \cite{hossain2021fair} and \cite{biswas2023fairness} define fairness for a multi-agent setting, but fairness with respect to the agents rather than actions is considered.

The notion of fairness, for actions, in the aforementioned works is modelled as a constraint rather than a desired outcome, with reward maximisation being the primary objective. In our work, we use the concept of fairness of exposure, introduced by \cite{wang2021fairness} for the single-agent setting, which is an objective-oriented notion of fairness that addresses the problem of starvation among actions. Additionally, it is important to highlight that no work has previously studied proportionality-based fairness in a federated bandit setting with respect to the actions.
To the best of our knowledge, our work is the premiere work to generalize fair contextual bandits into a federated setting, in addition to being the first work to simultaneously incorporate the notion of fairness and privacy for the bandit problem.  

%%%%%%%%%%%%%%%%%%%%%%%%%%%%%%%%%%%%%%%%%%%%%%%%%
\section{Model Preliminaries}\label{sec:model}
%%%%%%%%%%%%%%%%%%%%%%%%%%%%%%%%%%%%%%%%%%%%%%%%%

\subsection{Setting and Notations}
We abstract the problem as a federated contextual bandit setting where each of \(M=[m]\) agents are learning about actions \(a\in \altmathcal{D}\). 
The bandit algorithm runs for \(T\) rounds, where, at each round \(t\), an agent \(i\in M\) observes a context
vector \(\altmathcal{X}_t^i = (x_t^i(a))_{a \in \altmathcal{D}}\)\ \ (\(\mid\mid x_t^i(a) \mid\mid_2 \leq 1; \forall a\)) with \(x_t^i(a) \in \mathbb{R}^d\)
and selects an action \(a_{t}^{i}\). Each agent observes a different context vector and selects an action independently at each round \(t\). The agent obtains a reward for a selected action \(a_{t}^{i}\) at time \(t\) which we represent as \(y_{t}^i(a_{t}^{i}) = \theta^{*} \cdot x_{t}^i(a_{t}^{i}) + \eta_{t}(a_{t}^{i})\). Here, \({\theta^{*}}\in \mathbb{R}^{d}\) is an unknown but fixed parameter. As standard in the literature,  \(\eta_{t}(a_{t}^{i})\) is a noise parameter, which is i.i.d. sub-Gaussian with mean  \(0\). Thus, the expected reward for an action \(a\) at time \(t\), for an  agent \(i\), is given by \(\mathbb{E}[y_{t}^i(a)]  = \theta^{*} \cdot x_{t}^i(a)\). We denote this reward by the quantity \(\mu_a\mid\altmathcal{X}_t^i\) representing the expected reward for an action \(a\), when \(i^{th}\) agent is observing the context vector \(\altmathcal{X}_t^i\). Note that \(\theta^*\) (the true parameter) is the same for all the agents and is learned by the agents till time \(T\) in a collaborative fashion while preserving the privacy of their contexts/reward observations and satisfying the fairness guarantees.

We denote the set of available contexts to all the agents at time \(t\) as \(\altmathcal{X}_t=(\altmathcal{X}_t^i)_{i\in M}\), \(\altmathcal{X}^i=(\altmathcal{X}_t^i)_{t=1}^{t=T}\) and \(\altmathcal{X} = \{\altmathcal{X}^1,\altmathcal{X}^2,\ldots, \altmathcal{X}^m\}\). 
The goal of each agent \(i\) is to implement a policy \(\pi^i_t(\altmathcal{X}_t^i)\) which denotes the vector of probabilities of action selection by \(i^{th}\) agent at time \(t\). The probability of selecting action \(a\) is denoted by \(\pi^i_t(a, \altmathcal{X}_t^i)\). Instead of maximizing the reward,
each agent needs to ensure fairness amongst the actions so that all actions get a fair fraction of chances to avoid otherwise observed ``winner takes it all'' \cite{Mehrotra}  problem. Specifically, this setting aims to learn a policy that selects actions with probabilities proportional to their merit. Note that the objective here is to learn the fair policy rather than the optimal-reward policy.

Agents assign a merit score function $f^i$ over the actions based on their expected rewards for the given context.  \(f^i: \mathbb{R}^{+} \rightarrow \mathbb{R}^{+}\) where $f^i(\mu_a\mid \altmathcal{X}_t^i)$ denotes the score assigned by agent $i$ for the action $a$ when observed context is $\altmathcal{X}_t^i$. Each agent then needs to implement the policy such that the following fairness constraint, which is denoted as fairness of exposure, is satisfied:
\begin{equation}
    \label{eq:fair_cond}
    \frac{\pi^i_t(a, \altmathcal{X}_t^i )}{f^i(\mu_a\mid\altmathcal{X}_t^i)} = \frac{\pi^i_t(a^{\prime} , \altmathcal{X}_t^i)}{f^i(\mu_{a^{\prime}} \mid\altmathcal{X}_t^i)} \; \forall a, a^{\prime} \in \altmathcal{D}
\end{equation}

$f^i$ quantifies the utility of rewards derived from an arm for the agent. We assume Minimum merit and Lipschitz continuity properties on merit function \cite{wang2021fairness}. The minimum merit property provides a lower bound on the merit function, i.e. \(\min_{\mu} f^{i}(\mu) \geq \gamma\), \(\forall i \in M\) for some \(\gamma > 0\). Lipschitz continuity property assumes that the merit function is Lipschitz continuous, i.e., \(\forall \mu_{1}, \mu_{2}, i \in M\),  \(\lvert f^{i}(\mu_{1}) - f^{i}(\mu_{2}) \rvert \leq L \lvert \mu_{1} - \mu_{2} \rvert \) for some \(L > 0\).   

We denote the optimal policy by \(\pi_{*}^i(\altmathcal{X}_t^i)\) when \(\theta^*\) is known, i.e., at round \(t\), it satisfies fairness condition (Eq.~\ref{eq:fair_cond}). Note that given a context vector the optimal policy, \(\pi_{*}^i(.)\), does not depend on round \(t\), \(\pi_{*}^i(a, \altmathcal{X}_t^i) = \frac{f^i(\theta^* \cdot x_t^i(a))}{\sum_{a'\in \altmathcal{D}}f^i(\theta^* \cdot x_t^i(a'))}\). Typically, \(\theta^*\) being unknown, each agent is learning \(\theta^*\) and in turn the optimal policy through algorithm \(\altmathcal{A}\) over the rounds, taking actions using policy \(\pi_t^i(\cdot)\). \(\hat{\theta}_t^i\) is used to denote the learnt \(\theta^*\) for agent $i$ at time $t$. Unlike the optimal policy, \(\pi_t^i(\cdot)\) is round dependent. For agent \(i\) at round \(t\) the  \emph{instantaneous fairness regret} is defined as: $FR_t^i(\altmathcal{A}, \altmathcal{X}_t^i) = \sum_{a \in \altmathcal{D}} \lvert \pi_{*}^i(a, \altmathcal{X}_t^i)  - \pi_t^i(a, \altmathcal{X}_t^i)\rvert$.
As these agents learn about the same actions, they can communicate with each other about their estimates of \(\theta^{*}\) and learn it faster, reducing the per-agent fairness regret. We assume that all the agents deploy the same learning algorithm. Thus, we define the \emph{fairness regret} defined as:
\begin{definition}{Fairness Regret.}\label{def:fr} For a learning algorithm \(\altmathcal{A}\), we define fairness regret as \(FR(\altmathcal{A},T, \altmathcal{X}) = \frac{1}{m}\sum_{i\in M} FR^i(\altmathcal{A}, T, \altmathcal{X}^i)\) where \(FR^i(\altmathcal{A},T , \altmathcal{X}^i) = \sum_{t=1}^{T} FR_t^i(\altmathcal{A} ,\altmathcal{X}_t^i)\)
\end{definition}
Henceforth, we will avoid using \(\altmathcal{X}_t^i\) from fairness regret to avoid notation clutter. Additionally, since we are bounding it only for the algorithms in the paper, we refer to the above quantities as \(FR^i_t,FR_i, FR\). We also use
\(FR^i([T_1,T_2])\) to denote \( \sum_{t=T_1}^{t=T_2} FR^i_t\) and similarly \(FR([T_1,T_2])\).

\input{Subparts/Example}
\subsection{Fairness in Single-Agent Contextual MAB}
%%%%%%%%%%%%%%%%%%%%%%%%%%%%
We start with some notation and summarize FairX-LinUCB for a single-agent MAB setting \cite{wang2021fairness}.
\begin{itemize}
    \item If \(H\) is positive semi-definite matrix, it is represented by \(H \succeq 0\). Additionally, for two matrices \(H_1\) and \(H_2\), \(H_1 \succeq H_2\) implies \(H_1-H_2 \succeq 0\).
    \item The \(H-norm\) for vector \(y\) w.r.t. a positive semi-definite matrix \(H\), is denoted by \(\left\| y \right\|_{H} = \sqrt{y^\intercal  H y}\)
\end{itemize}

The central idea is to construct a confidence region, \(CR_t\), at every round \(t\), containing \(\theta^{*}\) with high probability. The confidence region is an ellipsoid centered around the linear regression estimate \(\hat{\theta}_{t} = ( I \lambda + X_{< t} {X_{< t}}^{\intercal})^{-1}  {X_{< t}}^{\intercal} Y_{< t}\). Here, \(X_{< t} = [{x_{1}(a_{1})}^{\intercal} \ldots {x_{t-1}(a_{t-1})}^{\intercal}]^{\intercal}\), \(Y_{< t} = [{y_{1}(a_{1})} \ldots y_{t-1}\) \((a_{t-1})]^{\intercal}\), and \(x_t(a_t)\) denotes the context of the selected action \(a_t\) at time \(t\). The proposed algorithm then optimistically selects \(\theta_{t}\) from the confidence region, and the selection policy, \(\pi_{t}\), using \(\theta_{t}\) based on the constraints. The selection policy defines a probability distribution over the actions, based on which an action is chosen, and the observed rewards for the chosen action are used to improve the estimation further. Optimistic selection is a non-convex-constrained optimization problem, and projected gradient descent is used to find approximate solutions.
%%%%%%%%%%%%%%%%%%%%%%%
\subsection{Privacy requirements}
%%%%%%%%%%%%%%%%%%%%%%%
We consider privacy over the agent-action interaction, i.e., for any agent \(i\), we consider that the context vectors (\(\altmathcal{X}^{i}\)) and the observed feedback (\((y_{t}^{i} (a_{t}^{i}))_{t \in [T]}\)) should be kept private. Considering that agent only needs to store \(x_{t}^{i} (a_{t}^{i})\) for feedback estimation, we use the differential privacy definition with respect \((x_{t}^{i} (a_{t}^{i}), y_{t}^{i} (a_{t}^{i}))_{t \in [T]}\). Our differential privacy notion matches the one defined in \cite{dubey2020differentially}. Here, we leverage their differential privacy definition for our setting. Let us consider two sets \(\altmathcal{S}_{i} = (x_{t}^{i} (a_{t}^{i}) , y_{t}^{i} (a_{t}^{i}))_{t \in [T]}\) and \({\altmathcal{S}_{i}}^{\prime} = ({x_{t}^{i} (a_{t}^{i})}^{\prime} , {y_{t}^{i} (a_{t}^{i})}^{\prime})_{t \in [T]}\). They are considered to be \(t^{\prime}-neighbors\) if at all time steps \(t \neq t^\prime\), \((x_{t}^{i} (a_{t}^{i}) , y_{t}^{i} (a_{t}^{i})) = ({x_{t}^{i} (a_{t}^{i})}^{\prime} , {y_{t}^{i} (a_{t}^{i})}^{\prime})\).

\begin{definition}{Federated Differential Privacy~\cite[Definition~1]{dubey2020differentially}}\label{def:dp} In a federated learning setting with \(m \geq 2\) agents, a randomized multi-agent contextual bandit algorithm \(\altmathcal{A} = (\altmathcal{A}^{i})_{i=1}^m\) is \((\epsilon,\delta, m)-\) federated differentially private under continual multi-agent observation if for any \(i,j \in M\) such that \(i \neq j\), any \(t\) and set of sequences \(\mathbb{S}_i = (\altmathcal{S}_k)_{k=1}^{m}\) and \({\mathbb{S}_i}^{\prime} = (\altmathcal{S}_k)_{k=1, k\neq i}^{m} \bigcup {\altmathcal{S}_{i}}^{\prime}\) such that \({\altmathcal{S}_{i}}^{\prime}\) and \({\altmathcal{S}_{i}}\) are \(t^{\prime}-\)neighbors, and any subset of actions \((a_{t}^{j})_{t \in [T]} \subset \altmathcal{D} \times \ldots \times \altmathcal{D}\) of actions, it holds that:
\begin{align*}
    \mathbb{P}(\altmathcal{A}^{j}(\mathbb{S}_{i}) \in (a_{t}^{j})_{t \in [T]}) \leq e^{\epsilon} . \mathbb{P}(\altmathcal{A}^{j}({\mathbb{S}_{i}}^{\prime}) \in (a_{t}^{j})_{t \in [T]}) + \delta
\end{align*}
\end{definition}

\noindent Here, the quantity, $\altmathcal{L}_{A^j(\mathbb{S}_i)\mid\mid A^j(\mathbb{S}_i^\prime)}^{o} = \log(\frac{\mathbb{P}(\altmathcal{A}^{j}(\mathbb{S}_{i}) \in o)}{ \mathbb{P}(\altmathcal{A}^{j}({\mathbb{S}_{i}}^{\prime}) \in o)})$ refers to privacy loss incurred by observing output \(o = (a_{t}^{j})_{t \in [T]}\).
\paragraph{\textbf{Goal:}} Each agent's goal is to learn \(\theta^*\) while minimizing fairness regret (Definition \ref{def:fr}); albeit ensuring differential privacy guarantees (Definition \ref{def:dp}).  

%Next section proposes a non-private algorithm, \algoname, followed by its privacy-guaranteeing version \ourprivalgo.
%%%%%%%%%%%%%%%%%%%%%%%%%%%%%%%%%%%%%%%%%%%%%%%%%
\section{Multi-Agent Fair and Private Contextual Bandit Algorithm}
%%%%%%%%%%%%%%%%%%%%%%%%%%%%%%%%%%%%%%%%%%%%%%%%%
% \sj{This para needs to be written better. First line is not clear, it is not clear what is meant by privacy budget? I feel there is sudden jump.}
The communication protocol currently used in federated bandits literature is not suitable for achieving bounded fairness regret. It is important to limit the number of communication rounds and maintain a constrained gap between communication instances in order to ensure both bounded fairness regret, and scalability with private methods. The total privacy loss, which is the composition of privacy losses incurred overall communication rounds, is proportional to the number of communication rounds. Thus, it follows that for a budgeted (fixed) total privacy loss, the maximum possible per-round privacy loss is inversely proportional to the number of communication rounds. As a result, the number of communication rounds should be bounded to control the accumulation of noise and maintain privacy within acceptable limits. At the same time, bounding the gaps between communication rounds is necessary to make fairness regret claims. 

In this section, we firstly build an algorithm, \algoname, that learns \(\theta^*\) collectively amongst \(m\) agents using a novel communication protocol. We then design a privacy-preserving version, \ourprivalgo, in Section~\ref{ssec:priv_algo}.

%%%%%%%%%%%%%%%%%%%%%%%
\subsection{\algoname} \label{subs:Multi-Agent Fair Contextual Bandit Algorithm}
\input{Subparts/algo}

We consider a group of \(m\) agents actively participating in the contextual bandit problem and maintaining synchronization through periodic communication. Algorithm ~\ref{alg: Fed-fairX} without the privatizer routine represents \algoname . Essentially, the exact information of the agents is sent to other agents when communication is required. For any agent \(i\), at round \(t\), let the last synchronization round take place at instant \(t^{\prime}\). Then, there exist two sets of parameters. The first set of parameters is the set of all observations made by all \(m\) agents till round \(t^{\prime}\). We store this in terms of a shared gram matrix, \(U_{t} = \sum_{i \in M} ( \lambda I + \sum_{\tau = 1}^{t^{\prime}} (x_{\tau}^{i}(a_{\tau}^{i})) (x_{\tau}^{i}(a_{\tau}^{i}))^{\intercal})\), and a shared vector, \(u_{t} = \sum_{i \in M} \sum_{\tau = 1}^{t^{\prime}} (x_{\tau}^{i} (a_{\tau}^{i})) y_{\tau}^{i} (a_{\tau}^{i})\). Secondly, each agent has access to its own observations since the last communication round. We note those using the gram matrix \(S_{t}^{i} = \sum_{\tau = t^{\prime}}^{t} (x_{\tau}^{i}(a_{\tau}^{i})) (x_{\tau}^{i}(a_{\tau}^{i}))^{\intercal}\) and the reward vector \(s_{t}^{i} = \sum_{\tau = t^{\prime}}^{t} (x_{\tau}^{i} (a_{\tau}^{i})) y_{\tau}^{i} (a_{\tau}^{i})\), where \(t^{\prime}\) was the last communication round. The agents use combined parameters for estimating the linear regression estimate, \(\hat{\theta}_{t}^{i}\). For an agent \(i\), 
\(V_{t}^{i} =  U_{t} + S_{t}^{i}, \ b_{t}^{i} = u_{t} + s_{t}^{i}, \hat{\theta}_{t}^{i} =  (V_{t}^{i})^{-1}  b_{t}^{i}\). 
The agents then constructs a confidence region, \(CR_t^{i}\) around \(\hat{\theta}_{t}^{i}\).  Suitable sequence \([\sqrt{\beta_{t}^{i}}]_{i\in M,t \in [T]}\) needs to be used, ensuring that with high probability \(\forall i,t,\ \theta^* \in CR_{t}^{i}\). An optimistic estimate, \(\theta_{t}^{i}\) is selected from \(CR_t^{i}\) (line 6 of Algorithm. ~\ref{alg: Fed-fairX}). The agent selects the action using a policy construction, \(\pi_{t}^{i}\). This ensures fairness by assigning a probability distribution for action selection based on estimated merit. We now explain our communication protocol that achieve sub-linear fairness regret.

% {\textbf{Communication Protocol.} The agents communicate periodically to share their observations in order to accelerate the learning process. We propose that the agents communicate with increasing separation (separation doubling each time) between two communication rounds in the first \(\ceil{\frac{T}{md^{2}\log^{2}{(1+T/d)}}}\) rounds (line 16-17 of Algorithm ~\ref{alg: Fed-fairX}) and only communicate after every \(\ceil{\frac{T}{md^{2}\log^{2}{(1+T/d)}}}\) rounds thereafter. We note that the number of communication rounds using the above protocol is upper bounded by \(\ceil{2md^{2}\log^{2}{(1+T/d)}}\). 

%\sj{It will be better if we contrast this communication protocol with existing communication protocol like sujit suggested.} 

%\sj{If we communicate at all times, then how can we extend fair algorithm to multi-agent setting? Is it trivial? If yes, then say that first and then say that however communicating with all agents at all times lead to communication efficiency, therefore it is important to design a novel communication protocol.}

\textbf{Communication Protocol.} If the agents were to communicate in every round without any optimization, they could enhance their fairness regret by order of \(O(1/\sqrt{m})\). However, communicating at every round results in inefficiencies and potential privacy breaches. To address these concerns, our algorithm suggests a communication strategy allowing agents to communicate only \(\ceil{2md^{2}\log^{2}{(1+T/d)}}\) times while achieving comparable fairness regret performance. In our proposed approach, we suggest that the agents communicate with increasing intervals between two consecutive communication rounds during the first \(\ceil{\frac{T}{md^{2}\log^{2}{(1+T/d)}}}\) rounds (line 12-13 of Algorithm ~\ref{alg: Fed-fairX}). Subsequently, they communicate only after every \(\ceil{\frac{T}{md^{2}\log^{2}{(1+T/d)}}}\) rounds. Rapid communication in the initial rounds proves beneficial in practice, considering the trend in regret is sublinear in \(T\). Concurrently, the number of communication rounds and the gap between the communication rounds remain bounded. This distinguishes it from the communication protocols employed by \cite{dubey2020differentially,solanki2022differentially}, where the gaps between communication rounds can be of the order \(O(T)\), which makes it difficult to bound fairness regret.
In summary, on observing the context set, each agent utilizes their estimate of \(\theta^{*}\) to formulate a selection policy, which yields a probability distribution for choosing an action. Once an action is selected and the corresponding reward is observed, the agents update their local estimates and periodically exchange these updates with each other to enhance the accuracy of the shared estimates.
%\sa{We refrain from explicitly giving an algorithmic formulation of \algoname, the non-private algorithm, due to space constraints. However, we note that Algorithm ~\ref{alg: Fed-fairX} represents the \algoname on removing the privatizer routine (discussed in Section~\ref{ssec:priv_algo}).}
% For non-private agents, we propose that the agents share the gram matrices and the reward vectors constructed using their observations, with every other agent, allowing them to construct the shared gram matrix and shared vectors.
%%%%%%%%%%%%%%%%%%%%%%%
\subsection{\ourprivalgo}\label{ssec:priv_algo}

The key difference between \ourprivalgo\ and \algoname\ lies in the communication perturbation. In a non-private setting, we communicate exact observations about context and reward to all other agents. However, we must carefully add perturbation for the private setting to satisfy the differential privacy constraints mentioned in section ~\ref{sec:model}. In the private setting, let \(\hat{U}_{t}^{i} = \sum_{\tau = 1}^{t-1} (x_{\tau}^{i}(a_{\tau}^{i})) (x_{\tau}^{i}(a_{\tau}^{i}))^{\intercal} + H_{t}^{i},\ \hat{u}_{t}^{i} = \sum_{\tau = 1}^{t-1} (x_{\tau}^{i} (a_{\tau}^{i})) y_{\tau}^{i} (a_{\tau}^{i}) + h_{t}^{i}\) denote the perturbed contexts and rewards. Here \(H_{t}^{i}\) and \(h_{t}^{i}\) are noise additions used for perturbation. Here, \(V_{t}^{i} = \sum_{i\in M} \Hat{U}_{t}^{i} + S_{t}^{i}\) and \(b_{t}^{i} = \sum_{i\in M} \Hat{u}_{t}^{i} + s_{t}^{i}\), where \(S_{t}^{i}\) and \(s_{t}^{i}\) remains same as stated in Section~\ref{subs:Multi-Agent Fair Contextual Bandit Algorithm}. We note that \(V_{t}^{i}\) can also be represented as: $V_{t}^{i} = G_{t}^{i} + H_{t}^{i}$ with \(G_{t}^{i}\) denoting the gram matrix in absence of noise perturbations.

To achieve privacy, we introduce a privatized version of the synchronization process amongst the agents. We do so by using the privatizer routine, which uses a tree-based mechanism to communicate while limiting the noise addition. The tree-based mechanism for differential privacy maintains a binary tree of logarithmic depth in terms of communication rounds. The sequential data released at communication rounds are stored at the leaf nodes, while every parent node stores the sum of the child nodes' data. In addition, noise is sampled at each node to maintain privacy. This allows for returning partial sums by adding at max \(k\) nodes if \(k\) was the depth of the tree. While our algorithm vastly differs from the FedUCB algorithm~\cite{dubey2020differentially} in terms of objective constraint, arm selection protocol, and communication round selection, it resembles our algorithm in terms of linear regressor estimation in a federated setting. Based on this, we can use the privatizer routine with marginal changes to ensure privacy guarantees. The privatizer routine is formally outlined for completeness.  

\input{Subparts/privatizer}

%In the next section, we provide theoretical guarantees for \algoname\ and  \ourprivalgo . 

%%%%%%%%%%%%%%%%%%%%%%%%%%%%%%%%%%%%%%%%%%%%%%%%%
\section{Theoretical Analysis} \label{sec:theory_analysis}
On a high level, the fairness regret proof considers a single hypothetical agent who plays \(mT\) rounds instead of considering \(m\) agents playing \(T\) rounds, each with sparse communication. The bounded deviation from this scenario to our intended setting is used to show the fairness regret analysis. Lemma \ref{lemma:Elliptical Potential} captures the fairness regret in terms of the determinant of the gram matrices, which is important to capture the deviation between the hypothetical agent and our intended set of agents, while lemma \ref{lemma:MDS} is useful for fairness regret bounds for a single-agent. Lemma \ref{lemma:Instanteous Regret} formalizes the instantaneous fairness regret, a prerequisite for proving Theorem \ref{theorem: fairness regret}.
%%%%%%%%%%%%%%%%%%%%%%%%%%%%%%%%%%%%%%%%%%%%%%%%%
\input{Subparts/proof}

%Having proved the theoretical robustness of \algoname\ and \ourprivalgo, we validate their efficacy through simulation-based experiments in the next section. 

%%%%%%%%%%%%%%%%%%%%%%%%%%%%%%%%%%%%%%%%%%%%%%%%%
% \input{Plots/Epsilon} 

\section{Experimental Analysis}\label{sec:exp_analysis}

% In this section, we evaluate the performance of proposed algorithms with extensive simulations. %We begin with explaining the experimental setup, then note our observations and analyze the results obtained. 

%%%%%%%%%%%%%%%%%%%%%%%%%%%%%%%%%%%%%%%%%%%%%%%%%
%%%%%%%%%%%%%%%%%%%%%%%%%%%
\input{Plots/Round}

\subsection{Experimental Set-up}
\paragraph{Dataset}
Synthetic datasets were generated for all experiments by randomly fixing the model parameter  \(\theta^{*}\). Context size was set to five ($d=5$), and feature vectors \(\altmathcal{X}_{t}^{i}\) were sampled from a uniform distribution, \(x_{t}^{i} (a) \in [0,1]^d\). Noise  \(\eta_{t}^{i} (a)\), sampled from a normal distribution centered at \(0\), was added to produce reward observations.

\paragraph{Merit Function and Optimization}
A steep merit function,\(f(\cdot) = e^{10 \mu}\), was employed, similar to \cite{wang2021fairness}. Projected gradient descent was used in each round to solve the resulting non-convex optimization problem.

\subsection{Evaluation Set-up}

\paragraph{Evaluation Metric}
Fairness regret was used as the primary evaluation metric to assess the algorithms' ability to balance performance and fairness. Exp \(1\) and \(2\) shows fairness regret trends with respect to rounds while Exp \(3\) and \(4\) uses the fairness regret at \(t = 100,000\). The objective is to minimise fairness regret, and thus it is being used as the evaluation metric in the experiments. (Though our focus is on fairness, for completeness, we also evaluate the proposed algorithms for reward regret~\cite{wang2021fairness} in Appendix.)

\paragraph{Experiment Repetition}
All reported results were averaged over \(5\) runs to ensure statistical significance.
\subsection{Baselines}
As we propose a novel setting, there are no algorithms for direct comparisons. \(2\) different kinds of baselines are used to demonstrate the efficacy of our proposed algorithm. 

\paragraph{Single-Agent Baseline (\(B0\))}
FairX-LinUCB algorithm was employed as a single-agent baseline to facilitate comparison with federated learning approaches. We note it as \(B0\) in our experiments. Each agent essentially learns on their own do not communicate with other agents under this baseline.

\paragraph{Communication Protocol Baseline (\(B1\) , \(B2\))}
Two existing communication protocols from \cite{dubey2020differentially} and \cite{solanki2022differentially} were compared against the proposed protocol to evaluate its efficacy. These have been termed \(B1\) and \(B2\) respectively. Note that the algorithms proposed in \cite{dubey2020differentially} and \cite{solanki2022differentially} optimize for traditional regret, hence \ourprivalgo\ has been modified to just use their proposed communication protocols to form \(B1\) and \(B2\).

\subsection{Experiments}
\paragraph{Exp 1: Single-Agent vs Federated Learning}
Compares the fairness regret of baseline \(B0\) to the proposed non-private algorithm, \algoname\, and its differentially private counterpart, \ourprivalgo, for \(10\) agents (\(m\)). [\(\epsilon = 2\), \(\delta=0.1\), \(t \in [1,100000]\)]

\paragraph{Exp 2: Communication Protocol}
Assesses the performance \ourprivalgo\ against \(B1\) and \(B2\) with \(10\) agents. [\(\epsilon = 2\), \(\delta=0.1\), \(t \in [1,100000]\)]

\paragraph{Exp 3: Dependence on \(m\)}
Compares the impact of the number of agents (\(m\)) on the fairness regret of both proposed algorithms. [\(\epsilon = 2\), \(\delta=0.1\), \(t=100000]\)]

\paragraph{Exp 4: Privacy Budget}
Examines the effect of the privacy budget (\(\epsilon\)) on the fairness regret of the private algorithm. [\(m=10\), \(\delta=0.1\), \(t=100000]\)]

\subsection{Inferences}
\begin{itemize}
    \item Both federated learning algorithms outperformed the single-agent baseline in terms of fairness regret.
    \item \ourprivalgo\ outperforms B1 while producing comparable performance for B2. But unlike B2, \ourprivalgo\ has bounded communication gaps, which is necessary for the theoretical guarantees provided. In B2, communication gaps are as high as $O(T)$ in the later stages, and hence, in theory, fairness regrets could be as bad as $O(T)$ for B2. 
    \item The fairness regret scales as expected with respect to the number of agents, validating theoretical results.
    \item The private algorithm achieved reasonable performance for \(\epsilon\) values of \(1\) or greater, highlighting the trade-off between privacy and regret.
\end{itemize}
%%%%%%%%%%%%%%%%%%%%%%%%%%%%%%%%%%%%%%%%%%%%%%%%%%%%%%%%%%%%%%%%%%%%%%%%

% \bibliographystyle{ACM-Reference-Format} 
\bibliography{neurips_2023}

%%%%%%%%%%%%%%%%%%%%%%%%%%%%%%%%%%%%%%%%%%%%%%%%%%%%%%%%%%%%%%%%%%%%%%%%

\newpage

\begin{center}
      \Large\textbf{Appendix}\\
\end{center}
   
\section*{Proofs}

\begin{lemma} \label{lemma:Instanteous Regret} (Lemma 3 in main text)
    For the \algoname, with high probability, the instantaneous regret for any agent \(i\) is bounded by,
    \begin{align*}
        {FR}_{t}^{i} = \sum_{a \in \altmathcal{D}} \left| \pi_{t}^{i} - \pi_{*}^{i} \right| \leq \frac{4L\sqrt{\beta_{t}}}{\gamma} \mathbb{E}_{a \sim \pi_{t}^{i} } \left\| x_{t}^i (a) \right\|_{(V_t^{i})^{-1}}
    \end{align*}
\end{lemma}

\begin{proof} 
\allowdisplaybreaks
\begin{align*}
    {FR}_{t}^{i} &= \sum_{a\in \altmathcal{D}} \left| \frac{f^i(\theta^{*}x_{t}^i (a))}{\sum_{a^{'} \in \altmathcal{D}} f^i(\theta^{*}x_{t}^i (a^{\prime}))} - \frac{f^i(\theta^{i}_{t}x_{t}^i (a))}{\sum_{a^{'} \in \altmathcal{D}} f^i(\theta^{i}_{t}x_{t}^i (a^{\prime}))} \right| \\ \\
    &= \sum_{a} \left| \frac{ \splitfrac{f^i(\theta^{*}x_{t}^i (a)) \sum_{a^{'}} f^i(\theta^{i}_{t}x_{t}^i (a^{\prime}))}{ - f^i(\theta^{i}_{t}x_{t}^i (a)) \sum_{a^{'}} f^i(\theta^{*}x_{t}^i (a^{\prime})) }}{\sum_{a^{'}} f^i(\theta^{i}_{t}x_{t}^i (a^{\prime})) \sum_{a^{'}} f^i(\theta^{*}x_{t}^i (a^{\prime}))} \right| \\
    & \\
    &= \sum_{a}  \frac{ \left| \splitfrac{f^i(\theta^{*}x_{t}^{i} (a)) \sum_{a^{'}} \left( f^i(\theta^{i}_{t}x_{t}^i (a^{\prime})) -  f^i(\theta^{*}x_{t}^i (a^{\prime})) \right)}{+\left( f^i(\theta^{*}x_{t}^{i} (a)) - f^i(\theta^{i}_{t}x_{t}^{i} (a)) \right) \sum_{a^{'}} f^i(\theta^{*}x_{t}^i (a^{\prime}))} \right|}{\sum_{a^{'}} f^i(\theta^{i}_{t}x_{t}^i (a^{\prime})) \sum_{a^{'}} f^i(\theta^{*}x_{t}^i (a^{\prime}))}  \\
    & \\
    &\leq \sum_{a}  \frac{  \splitfrac{ \left|f^i(\theta^{*}x_{t}^{i} (a)) \sum_{a^{'}} \left( f^i(\theta^{i}_{t}x_{t}^i (a^{\prime})) -  f^i(\theta^{*}x_{t}^i (a^{\prime})) \right) \right|}{+\left|\left( f^i(\theta^{*}x_{t}^{i} (a)) - f^i(\theta^{i}_{t}x_{t}^{i} (a)) \right) \sum_{a^{'}} f^i(\theta^{*}x_{t}^i (a^{\prime})) \right|} }{\sum_{a^{'}} f^i(\theta^{i}_{t}x_{t}^i (a^{\prime})) \sum_{a^{'}} f^i(\theta^{*}x_{t}^i (a^{\prime}))} \\ \\
    &\leq  \frac{ 2\sum_{a} \left|  f^i(\theta^{*}x_{t}^{i} (a)) - f^i(\theta^{i}_{t}x_{t}^{i} (a))  \right|}{\sum_{a^{'}} f^i(\theta^{i}_{t}x_{t}^i (a^{\prime}))} \\ \\
    &= 2\sum_{a} \frac{\pi_{t}^{i}}{f^i(\theta^{i}_{t}x_{t}^{i} (a))} \left|  \begin{aligned}
        \splitfrac{f^i(\theta^{*}x_{t}^{i} (a)) - f^i(\hat{\theta^{i}_{t}}x_{t}^{i} (a))}{ + f^i(\hat{\theta^{i}_{t}}x_{t}^{i} (a)) - f^i(\theta^{i}_{t}x_{t}^{i} (a)) }
    \end{aligned} \right| \\ \\
    &\leq \frac{2L}{\gamma} \mathbb{E}_{a  \sim \pi_{t}^{i} } \left[ 
    \begin{aligned}
        \splitfrac{\left\| \theta^{*} - \hat{\theta^{i}_{t}} \right\|_{V_t^{i}} \left\| x_{t}^{i} (a) \right\|_{(V_t^{i})^{-1}} }{ + \left\| \hat{\theta^{i}_{t}} - \theta^{i}_{t} \right\|_{V_t^{i}} \left\| x_{t}^{i} (a) \right\|_{(V_t^{i})^{-1}} }
    \end{aligned} \right] \\
    &\leq \frac{4L\sqrt{\beta_{t}}}{\gamma}\mathbb{E}_{a \sim \pi_{t}^{i} } \left\|x_{t}^i (a)\right\|_{(V_t^{i})^{-1}}
\end{align*}
\end{proof}

\begin{theorem} \label{theorem: fairness regret} (Theorem 1 main text)
With high probability, \algoname\ achieves a fairness regret of  \(O\left(\frac{4\nu L\sqrt{\beta_{t}}}{\gamma}\sqrt{mTd\log{(1+\frac{T}{d})} + m^{2}d^{3}\log^{3}{(1+\frac{T}{d})}}\right)\)  when \(\mid\mid x_t^i(a)\mid\mid_2 \leq 1 \:\forall a,t,i\). 
\end{theorem}

\begin{proof}

\noindent Consider a hypothetical agent denoted by index \(0\) who plays in the following \(mT\) rounds - \((1,1), (1,2), \ldots (1,m), \ldots, (T,m)\) sequentially. Let the gram matrix for agent \(0\) till round \((\tau,j)\) be given by \(V^{0}_{(\tau,j)} = mI + \sum_{i=1}^{i=m} \sum_{t=1}^{\tau-1} (x_{t}^i (a^{i}_{t})) (x_{t}^i (a^{i}_{t}))^T \) \(+ \sum_{i=1}^{i=j} (x_{\tau}^i (a^{i}_\tau)) (x_{\tau}^i (a^{i}_\tau))^T \). Substituting \(U_1 = mI\) and \(L=1\) in Lemma \ref{lemma:Elliptical Potential} we get, 
\begin{align*}
    \sum_{i=1}^{i=m} \sum_{t=1}^{T} \left\|x_{t}^{i} (a_{t}^{i})\right\|_{(V_{(t,i)}^{0})^{-1}}^2 \leq 2d \log\left(1 + \frac{T}{d}\right)
\end{align*}

\noindent Let the communication in the original algorithm occur at rounds \(T_{1}, T_{2}, \ldots, \) \( T_{p-1}\). Let \(\Psi_{k} = mI + \sum_{i=1}^{i=m} \sum_{t=1}^{T_{k}} (x_{t}^{i} (a_{t}^{i})) (x_{t}^{i} (a_{t}^{i}))^T\) be the synchronised gram matrix after communication round \(k\). 
Then \(\operatorname{det} \Psi_{0} = (m)^{d}\) and \(\operatorname{det} \Psi_{p} \leq \left(\frac{\operatorname{tr} (\Psi_{p})}{d}\right)^{d} \leq (m+mT/d)^{d}\). Thus, for any \(\nu > 1\), \(\log_{\nu}\left( \frac{\operatorname{det}(\Psi_{p})}{\operatorname{det} (\Psi_{0})} \right) \leq d \log_{\nu}(1 + \frac{T}{d})\). Let event \(E\) represent the set of rounds when \(1 \leq \frac{\operatorname{det}(\Psi_{k})}{\operatorname{det} (\Psi_{k-1})} \leq \nu\) is true. Then, in all but \(\ceil{ d \log_{\nu}(1 + \frac{T}{d}) }\) rounds \(E\) is true.

For any \(T_{k-1} \leq t \leq T_{k}\), when \(E\) is true,
\allowdisplaybreaks
\begin{align*}
\allowdisplaybreaks
    fr_{t}^{i} &\leq \frac{4L\sqrt{\beta_{t}}}{\gamma}\mathbb{E}_{a \sim \pi_{t}^{i} } \left\|x_{t}^{i} (a)\right\|_{(V_t^{i})^{-1}}  \\ 
    &\leq \frac{4L\sqrt{\beta_{t}}}{\gamma}\mathbb{E}_{a \sim \pi_{t}^{i} } \left\|x_{t}^{i} (a)\right\|_{(V^{0}_{(t,i)})^{-1}} \sqrt{\frac{\operatorname{det}V^{0}_{(t,i)} }{\operatorname{det}V^{i}_{t}}} \\
    &\leq \frac{4L\sqrt{\beta_{t}}}{\gamma}\mathbb{E}_{a \sim \pi_{t}^{i} } \left\|x_{t}^{i} (a)\right\|_{(V^{0}_{(t,i)})^{-1}} \sqrt{\frac{\operatorname{det} \Psi_{k} }{\operatorname{det}\Psi_{k-1} }} \\ 
    &\leq \frac{4\nu L\sqrt{\beta_{t}}}{\gamma}\mathbb{E}_{a \sim \pi_{t}^{i} } \left\|x_{t}^{i} (a)\right\|_{(V^{0}_{(t,i)})^{-1}}
\end{align*}
Here, second last equation follows because \(V^{i}_{t} \succeq \Psi_{k-1} \) and \( \Psi_{k} \succeq V^{0}_{(t,i)}\). \ Now, using Lemma 1 (main text),  
\allowdisplaybreaks
\begin{align*}
    &\sum_{i=1}^{m} \sum_{t\in E} fr_{t}^{i} \leq \sum_{i=1}^{m} \sum_{t\in E} \frac{4\nu L\sqrt{\beta_{t}}}{\gamma}\mathbb{E}_{a \sim \pi_{t}^{i} } \left\|x_{t}^{i} (a)\right\|_{(V^{0}_{(t,i)})^{-1}} \\
    & \leq \frac{4\nu L\sqrt{\beta_{T}}}{\gamma}\left( \sum_{i=1}^{m} \sum_{t=1}^{T} \left\|x_{t}^{i} (a_{t}^{i})\right\|_{(V^{0}_{(t,i)})^{-1}} 
        +  \sqrt{2mT\log{(4/\delta)}}\right) \\
     &\leq \frac{4\nu L\sqrt{mT\beta_{T}}}{\gamma} \left( \sqrt{ d \log(1 + \frac{T}{d}) } + \sqrt{2\log{(4/\delta)}} \right) \\
\end{align*}
Now, let us consider any period \(t \in [T_{k-1},T_{k}]\), where E does not hold and \(t_{k} = T_{k} - T_{k-1}\) represent the length of the interval. Fairness regret during this period is given by,
\allowdisplaybreaks
\begin{align*}
    &FR([T_{k-1},T_{k}]) %&= \sum_{i=1}^{m} \sum_{t= T_{K-1}}^{T_{k}} fr_{t}^{i} \\
    \leq \frac{4L\sqrt{\beta_{T}}}{\gamma} \sum_{i=1}^{m} \sum_{t= T_{K-1}}^{T_{k}} \mathbb{E}_{a \sim \pi_{t}^{i} } \left\|x_{t}^{i} (a)\right\|_{(V_t^{i})^{-1}}\\
&\leq \frac{4L\sqrt{\beta_{T}}}{\gamma} \left( \sum_{i=1}^{m} \sum_{t= T_{K-1}}^{T_{k}} \left\|x_{t}^{i} (a_{t}^{i})\right\|_{(V_t^{i})^{-1}} + m\sqrt{2t_{k}\log(4/\delta)}
    \right) \tag*{(Using Lemma 1 (main text))} \\
    &\leq \frac{4L\sqrt{\beta_{T}}}{\gamma}  \left( \sum_{i=1}^{m} \sqrt{t_{k} \log_{\nu}{\frac{\operatorname{det} V_{T_{k-1} + t_{k}}^{i}}{\operatorname{det}V_{T_{k-1}}^{i}}} } + m\sqrt{2t_{k}\log(4/\delta)}  \right) \\
\end{align*}
We know that \( \forall \) agents, \(t_{k} \leq \frac{T}{md^{2}\log^{2}{(1+T/d)}} + 1\) (otherwise there be a communication round), thus \(FR([T_{k-1},T_{k}]) \leq \frac{4L\sqrt{\beta_{T}}}{\gamma}  \) 

\( \left( \sqrt{\frac{m(T+md^{2}\log^{2}{(1+\frac{T}{d})})}{d\log{(1+\frac{T}{d})}} }+ \sqrt{\frac{2m(T+md^{2}\log^{2}{(1+\frac{T}{d})})}{d^{2}\log^{2}{(1+\frac{T}{d})}}\log(\frac{4}{\delta})}\right)\).
\if 0
\begin{align*}
    \leq \frac{4L\sqrt{\beta_{t}}}{\gamma}  \left[ \begin{aligned}
        &\sqrt{\frac{m(T+md^{2}\log^{2}{(1+T/d)})}{d\log{(1+T/d)}}} \\
        &+ \sqrt{2\frac{m(T+md^{2}\log^{2}{(1+T/d)})}{d^{2}\log^{2}{(1+T/d)}}\log(4/\delta)}
    \end{aligned} \right] 
\end{align*} \fi

\noindent Using the fact that \(E\) does not hold true in at most in \(\ceil{ d \log_{\nu}(1 + \frac{T}{d}) }\) rounds, we get  

\begin{align*}
    FR(T) \le O\left(\frac{4\nu L\sqrt{\beta_{T}}}{\gamma}\sqrt{
    mTd\log{(1+T/d)} + m^{2}d^{3}\log^{3}{(1+T/d)}}\right)
\end{align*}
\end{proof}

\begin{theorem} \label{theorem: private fairness regret} (Theorem 2 main text) With high probability,  when \(\mid\mid x_t^i(a)\mid\mid_2 \leq 1 \:\forall a,t,i\) and Lemma~\ref{lemma:beta_potential} holds, \ourprivalgo\ achieves a fairness regret of \\ 
\(O\left(\frac{4\nu L\sqrt{\beta_{T}}}{\gamma}\sqrt{mTd\log{(\frac{\bar{\rho}}{\underline{\rho}} + \frac{T}{d\underline{\rho}})} + m^{2}d^{3}\log^{3}{(\frac{\bar{\rho}}{\underline{\rho}} + \frac{T}{d\underline{\rho}})}}\right)\).

\end{theorem}
\begin{proof}
    We note that the proof follows from the proof of Theorem~\ref{theorem: fairness regret} with minor changes. The regularisation of \(\Psi_{k}\) is done using \(m\underline{\rho} I\) instead of \(mI\). This allows for a tight bound on \(\log_{\nu}\left( \frac{\operatorname{det} (\Psi_{p})}{\operatorname{det}(\Psi_{0})} \right)\) with appropriate values of \(\bar{\rho}\) and \(\underline{\rho}\). In addition, the property \(V_{t}^{i} \succeq G_{t}^{i} + M\underline{\rho}I\), is important for stating that \(fr_{t}^{t} \leq \frac{4\nu L\sqrt{\beta_{t}}}{\gamma}\mathbb{E}_{a \sim \pi_{t}^{i} } \left\|x_{t}^{i} (a)\right\|_{(V^{0}_{(t,i)})^{-1}}\) when \(E\) holds true. The rest of the proof follows similar to the proof of Theorem~\ref{theorem: fairness regret}.
\end{proof}

\newpage
\section*{Additional Experiments}

For completeness, we provide evaluation of our proposed algorithms for reward regret (defined in \cite{wang2021fairness}). The same experiments (Exp1, Exp2, Exp3 and Exp4), as described in Section 6 of main text, are performed and plotted for reward regret.

\input{Plots/Round-Appendix}

% \bibliographystyle{named}
% \bibliography{ijcai24}

\end{document}

%% file: Subparts/Example.tex
% \section{Toy Example}
\subsection{Why fairness of exposure?}
We motivate with a single agent setting who is interested in assigning tasks to 3 workers with unknown completion times. Let the optimal task assignment (according to Eq. 1)distribution be $[0.14, 0.28, 0.56]$, where faster worker is assigned more tasks, if the goal is to minimize total project completion time while ensuring exposure guarantees to the workers. %This is a contextual bandit problem, where we balance learning worker speeds with assigning tasks efficiently.
Traditional regret optimization finds the best worker which does not lead to balanced/fairer task allocation.%, but this may not be optimal when multiple tasks can be done concurrently. More balanced/fairer task allocation can lead to faster project completion.
While some approaches try to incorporate fairness into bandit algorithms, they often fall short in the task assignment scenario:

\begin{itemize}
    \item Delta-fairness~\cite{joseph2016fair,Shaarad-fair}, which prioritizes arms (workers) with higher rewards will essentially lead to giving maximum tasks to optimal (faster) worker, in this case the worker 3, however it does not provide any exposure guarantee.%, does not guarantee optimal completion time. For example, assigning all tasks to the fastest worker might not be efficient when others can work concurrently but satisfy the fairness constraints. %\sg{are we going to work with multiple pulls in this paper?I don't think so...so this explanation of concurrent pulls is not right here}
    \item Minimum share fairness~\cite{patil2020achieving,chen2020fair}ensures each worker receives a minimum fraction of tasks. Utility optimisation in this case relies on knowing expected completion times, which are unknown in our problem. This makes its effectiveness uncertain.
\end{itemize}   

In contrast, proportionality-based fairness offers a more promising approach by directly aligning fairness with utility optimization. Furthermore, when workers are involved in multiple projects simultaneously, (i.e., multiple agents are learning about the workers) federated learning with differential privacy can further optimize task assignment by sharing limited information privately, leading to faster learning and improved project completion times.%Proportionality-based fairness tackles this challenge by aiming for a distribution of task assignments proportional to expected rewards (inversely proportional to expected completion times). This incentivizes learning worker speeds while ensuring efficiency. In our example, the optimal probability distribution is around [0.14, 0.28, 0.56], favoring faster workers but distributing tasks to achieve faster overall completion compared to purely assigning to the fastest worker. This highlights the key distinction of proportionality-based fairness: it balances both fairness and reward optimization in this scenario, so as to optimise project utility.

%% file: Subparts/algo.tex
\begin{algorithm}[t!]
    \caption{\ourprivalgo}
    \begin{minipage}{\linewidth}
        \begin{algorithmic}[1] \label{alg: Fed-fairX}
            \STATE \textbf{Input:} \(\beta_{t}\), \([f^{i}]_{\forall i \in [m]}\), \(\lambda\), \(m\)
            %\FOR{$i = 1$ to $m$}
                \STATE \textbf{Initialization:} \(\forall i \in [m]\), \(V_{1}^{i} = S_{1}^{i} = U_{1} = \lambda\mathbf{I}_{d}\),  \(b_{1}^{i} = s_{1}^{i} = u_{1} = \mathbf{0}_{d}\),  \(\tau = 1\).
            %\ENDFOR
                \FOR{\(t = 1\) to \(T\)}
                    \FOR{\(i = 1\) to \(m\)}
                    \STATE Observe contexts \(\altmathcal{X}_{t}^{i}\);  \(\hat{\theta}_{t}^{i} = {(V_{t}^{i})}^{-1} b_{t}^{i}\);\ \(\mathbf{CR}_{t}^i = ( \theta : \norm{ \theta - (\hat{\theta}_{t}^{i}) }_{V_{t}^{i}} \leq \sqrt{\beta_{t}^{i}} )\)
                    \STATE \(\theta_{t}^{i} = \operatorname{\tiny{argmax}}_{\theta \in CR_{t}^{i}} \sum_{a\in \altmathcal{D}} \frac{f(\theta . x_{t}^{i}(a))}{\sum_{a^{'}\in \altmathcal{D}} f(\theta . x_{t}^{i}(a^{\prime}))} \theta . x_{t}^{i}(a)\)
                    \STATE Construct Policy \(\pi_{t}^{i}(a) = \frac{f(\theta_{t}^{i} . x_{t}^{i}(a))}{\sum_{a^{'}} f(\theta_{t}^{i} . x_{t}^{i}(a^{\prime}))}\)
                    \STATE Sample arm \(a_{t}^{i} \sim \pi_{t}^{i}\) and observe reward \(y_{t}^{i} (a_{t}^{i})\)
                    \STATE \(S_{t+1}^{i} = S_{t}^{i} + (x_{t}^{i} (a_{t}^{i})) (x_{t}^{i} (a_{t}^{i}))^{\intercal}\); \(s_{t+1}^{i} = s_{t}^{i} + (x_{t}^{i} (a_{t}^{i})) y_{t}^{i} (a_{t}^{i})\)
                    \IF{\(t == \tau\)}
                     \STATE \textit{Sync} \(\longleftarrow\) \textit{true}
                     \IF{\(t < \ceil{\frac{T}{md^{2}\log^{2}{(1+T/d)}}}\)}
                     \STATE \(\tau = 2\tau\)
                     \ELSE
                     \STATE \(\tau = \tau + \ceil{\frac{T}{md^{2}\log^{2}{(1+T/d)}}}\)
                     \ENDIF
                     \ENDIF
                %    \IF{$t < \ceil{\frac{T}{md^{2}\log^{2}{(1+T/d)}}}$ and $t == \tau$}
                 %       \STATE \textit{Sync} $\longleftarrow$ \textit{true}
                  %      \STATE $\tau += \tau$
                   % \ENDIF
                    %\IF{$\Delta_{t}^{i}  == \ceil{\frac{T}{md^{2}\log^{2}{(1+T/d)}}}$ }
                   %     \STATE \textit{Sync} $\longleftarrow$ \textit{true}
                    %\ENDIF
                    \IF{\textit{Sync}}
                        \STATE \([\forall j \in M]\) Send \(S_{t}^{j}, s_{t}^{j}  \rightarrow PRIVATIZER\) 
                        \STATE \([\forall j \in M]\) Receive \(\hat{U}_{t}^{j}, \hat{u}_{t}^{j} \leftarrow PRIVATIZER\)
                        \STATE \([\forall j \in M]\) Communicate \(\hat{U}_{t}^{j}, \hat{u}_{t}^{j}\) to others
                        \STATE \([\forall j \in M]\) \(U_{t+1} = \sum_{k =1 }^{M} \hat{U}_{t}^{k}\);  \(u_{t+1} = \sum_{k =1 }^{M} \hat{u}_{t}^{k}\);  \(S_{t}^{j} = \textbf{0}_{d \times d};\ s_{t}^{j} = \textbf{0}_{d};\ \Delta_{t}^{j} = 0\)
                        \STATE  \textit{Sync} \(\longleftarrow\) \textit{false}
                        
                    \ELSE
                        \STATE \(U_{t+1} = U_{t}^{i}\); \(u_{t+1} = u_{t}^{i}\); \( \Delta_{t+1}^{i} = \Delta_{t}^{i} + 1\)
                    \ENDIF                    
                    \STATE \(V_{t+1}^{i} = U_{t+1} + S_{t+1}^{i}\); \(b_{t+1}^{i} = u_{t+1} + s_{t+1}^{i}\)
                \ENDFOR
            \ENDFOR
        \end{algorithmic}
    \end{minipage}
\end{algorithm}

%% file: Subparts/privatizer.tex
\begin{algorithm}[t!]
    % \algsetup{linenosize=\small}
    % \scriptsize
    \caption{PRIVATIZER}
    \begin{minipage}{0.95\linewidth}
        \begin{algorithmic}[1] \label{alg: privatizer}
            \STATE \textbf{Input:} \(\epsilon, \delta, d, \tau\) (number of communication rounds), \(L\) (upper bound on norm of context vector)
            \STATE \textbf{Initialization:} 
            \STATE \(n = 1 + \ceil{\log \tau}\)
            \STATE \(\mathcal{T} \leftarrow \) a binary tree of depth \(n\)
            \FOR{each node \(i\) in \(\mathcal{T}\)}
                \STATE Create a noise matrix: \(\hat{N} \in \mathbb{R}^{d \times (d+1)}\), where \(\hat{N}_{kl} \sim \mathcal{N}(0,16n(L^{2} + 1)^{2} \log (2/\delta)^{2} /\epsilon^{2})\)
                \STATE \(N = (\hat{N} + \hat{N}^{\intercal})/\sqrt{2}\)
            \ENDFOR
            \STATE \textbf{Runtime:}
            \FOR{each communication round \(t\)}
                \STATE Receive \(S_{t}^{i}, s_{t}^{i}\) from agent, and insert it into \(\mathcal{T}\) as a \(d \times (d+1)\) matrix (Alg. 5, \cite{jain2012differentially})
                \STATE Receive \(M_{t}^{i}\) using the least nodes of \(\mathcal{T}\) (Alg. 5, \cite{jain2012differentially})
                \STATE \(\hat{U}_{t}^{i} = U_{t}^{i} + H_{t}^{i}\), top-left \(d\times d\) submatrix of \(M_{t}^{i}\)
                \STATE \(\hat{u}_{t}^{i} = u_{t}^{i} + h_{t}^{i}\), last column of \(M_{t}^{i}\)
                \STATE Return \(\hat{U}_{t}^{i}, \hat{u}_{t}^{i}\)
            \ENDFOR
        \end{algorithmic}
    \end{minipage}
\end{algorithm}

%% file: Subparts/proof.tex
\subsection{Regret Analysis}
The following lemma is useful in proving the fairness regret of \algoname.
\begin{lemma} \label{lemma:Elliptical Potential}
(Elliptical Potential~\cite[Lemma~22]{shariff2018differentially}). Let \({x}_1, \ldots, {x}_n \in R^d\) be vectors with each \(\left\| {x}_t\right\| \leq L\). Given a positive definite matrix 
\(U_1 \in R^{d \times d}\), define \(U_{t+1}:=U_t+x_t {x}_t^{\top}\) for all \(t\). Then
        $\sum_{t=1}^n \min \left\{1,\left\|x_t\right\|_{U_t^{-1}}^2\right\} \leq 2 \log \frac{\operatorname{det} U_{n+1}}{\operatorname{det} U_1} \leq 2 d \log \frac{\operatorname{tr} U_1+n L^2}{d \operatorname{det}^{1 / d} U_1}$
\end{lemma}
\noindent Also, we extend Lemma A.6.4 from ~\cite{wang2021fairness} to multi-agent setting as follows.
\begin{lemma} \label{lemma:MDS}
When \(\mid\mid x_t^i(a)\mid\mid_2 \leq 1 \:\forall a,t,i,\) for the \algoname \  algorithm, \(\forall i \in [m]\), with probability \(1-\delta/2\), 
\begin{align*}
    \left| \sum_{t=1}^{T} w_{t}^{i} (a_{t}^{i}) - \sum_{t=1}^{T} \mathbb{E}_{a \sim \pi_{t}^{i}} w_{t}^{i}(a)  \right| \leq \sqrt{2T \operatorname{ln}(4/\delta)}
\end{align*}
\end{lemma}

\noindent Here, \(w_{t}^{i}(a)= \sqrt{x_{t}^i (a) (V_{t}^{i})^{-1} (x_{t}^i (a))^{\intercal}}\) is the normalized width. With the help of the above lemmas, we now provide bounds on instantaneous regret \(FR_t^i\) and defer their proofs to appendix.
\begin{lemma} \label{lemma:Instanteous Regret}
    For the \algoname, with high probability, the instantaneous regret for any agent \(i\) is bounded by,
    \begin{align*}
        {FR}_{t}^{i} = \sum_{a \in \altmathcal{D}} \left| \pi_{t}^{i} - \pi_{*}^{i} \right| \leq \frac{4L\sqrt{\beta_{t}}}{\gamma} \mathbb{E}_{a \sim \pi_{t}^{i} } \left\| x_{t}^i (a) \right\|_{(V_t^{i})^{-1}}
    \end{align*}
\end{lemma}

\noindent The probability with which Lemma~\ref{lemma:Instanteous Regret} holds true is dependent on \(\beta_{t}^{i}\), where \(\beta_{t} = max_{i \in M} \beta_{t}^{i}\).

\begin{theorem} \label{theorem: fairness regret} With high probability, \algoname\ achieves a fairness regret of  \(O\left(\frac{4\nu L\sqrt{\beta_{t}}}{\gamma}\sqrt{mTd\log{(1+\frac{T}{d})} + m^{2}d^{3}\log^{3}{(1+\frac{T}{d})}}\right)\)  when \(\mid\mid x_t^i(a)\mid\mid_2 \leq 1 \:\forall a,t,i\). 
\end{theorem}

The values in sequence of \(\beta_{t}\) dictates the probability with which Lemma~\ref{lemma:Instanteous Regret}, and in turn Theorem~\ref{theorem: fairness regret} holds. The problem of selection of values in sequence of \(\beta_{t}\) is well studied in the literature. For instance, using Theorem 2 from ~\cite{NIPS2011_e1d5be1c}, it can be said that \(\theta^{*}\) lies in the confidence region with probability \(1- \alpha\) for \(\beta_{t} = O\left(d\log{(\frac{1+mt}{\alpha})} \right)\) resulting in a regret bounds of \(\Tilde{O} \left( d\sqrt{ mT\log^{2}{(1+mT/d)} } \right)\) for \algoname\  (typically \(m << T\) and hence the \(\sqrt{mT}\) term dominates $\sqrt{m^2\log^2(1+T/d)}$ ).

The key difference between private and non-private regret analysis lies in the gram matrix regularization and confidence interval construction (use of appropriate \(\beta_{t}\)). 

 We note the following claim is useful for completing \ourprivalgo 's regret analysis. It provides values for the sequence of \(\beta_{t}\) for which the confidence interval contains \(\theta^{*}\) with high probability.

\begin{lemma} \label{lemma:beta_potential}(Similar to~\cite[Proposition~2]{dubey2020differentially}) For an instance of problem where synchronisation occurs exactly \(n\) times in a span of \(T\) trials, and \(\underline{\rho}, \bar{\rho}\) and \(z\) are \((\alpha / 2nm)\)-accurate~\cite[Definition~3]{dubey2020differentially}. Then for \ourprivalgo\ with bounded target parameter (\(\mid\mid \theta^{*} \mid\mid_{2} \leq c\)), the sequence of \(\sqrt{\beta_{t}^{i}}\) is \((\alpha,M,T)\)-accurate if, $\sqrt{\beta_{t}^{i}} = \sigma \sqrt{2 \log{(\frac{2}{\alpha})} + d\log{(\frac{\bar{\rho}}{\underline{\rho}} + \frac{t}{d\underline{\rho}})}} + mc\sqrt{\bar{\rho}} + mz$
    
\end{lemma}

\begin{theorem} \label{theorem: private fairness regret} With high probability,  when \(\mid\mid x_t^i(a)\mid\mid_2 \leq 1 \:\forall a,t,i\) and Lemma~\ref{lemma:beta_potential} holds, \ourprivalgo\ achieves a fairness regret of \\ 
\(O\left(\frac{4\nu L\sqrt{\beta_{T}}}{\gamma}\sqrt{mTd\log{(\frac{\bar{\rho}}{\underline{\rho}} + \frac{T}{d\underline{\rho}})} + m^{2}d^{3}\log^{3}{(\frac{\bar{\rho}}{\underline{\rho}} + \frac{T}{d\underline{\rho}})}}\right)\).

\end{theorem}
% \begin{proof}
%     Deferred to app
% \end{proof}
% \begin{proof}
%     While we omit the complete proof here, we note that it follows from the proof of Theorem~\ref{theorem: fairness regret} with minor changes. The regularisation of \(\Psi_{k}\) is done using \(m\underline{\rho} I\) instead of \(mI\). This allows for a tight bound on \(\log_{\nu}\left( \frac{\operatorname{det} (\Psi_{p})}{\operatorname{det}(\Psi_{0})} \right)\) with appropriate values of \(\bar{\rho}\) and \(\underline{\rho}\). In addition, the property \(V_{t}^{i} \succeq G_{t}^{i} + M\underline{\rho}I\), is important for stating that \(fr_{t}^{t} \leq \frac{4\nu L\sqrt{\beta_{t}}}{\gamma}\mathbb{E}_{a \sim \pi_{t}^{i} } \left\|x_{t}^{i} (a)\right\|_{(V^{0}_{(t,i)})^{-1}}\) when \(E\) holds true. The rest of the proof follows similar to the proof of Theorem~\ref{theorem: fairness regret}.
% \end{proof}

% \begin{}

\subsection{Privacy Guarantees}

As mentioned in Sec. ~\ref{ssec:priv_algo}, we can leverage the privatizer routines to provide differential privacy guarantees for \ourprivalgo. At each synchronization, new observations, $S_{t}^{i}$ and $s_{t}^{i}$, are added to a leaf node, while all other nodes store the sum of the child nodes. Thus, $1 + \ceil{\log (n)}$ nodes of the tree, where $n$ is the total number of communication rounds, are sufficient to represent any partial sum till the last synchronization round. Since the privatizer routine follows the routine introduced by earlier works, it trivially follows that if each node guarantees \((\epsilon/\sqrt{8m \ln{(2/\delta)}},\) \(\delta/2m)-\)privacy, the outgoing communication is guaranteed to be $(\epsilon,\delta,m)-$federated differential private for each synchronization with similar values for \(\bar{\rho}, \underline{\rho}, z\). 
\begin{claim} (Follows from~\cite[Remark~3]{dubey2020differentially})
    The privatizer routine in \ourprivalgo\ guarantees that each of the outgoing messages for an agent \(i\) is \((\epsilon, \delta)-\)differentially private.
\end{claim}

%% file: Plots/Round.tex
% [3009.537541158756, 2022.1482217010616, 2161.8114335952987, 1383.9164497285542, 1632.1401226366554, 1224.2291907517565]

% [3027.814251030743, 3053.0032979190432, 2968.8011925361707, 2288.8568274001022, 2260.2942188295847, 1750.26074339728]

% [3145.1049759108437, 17766.264776730597, 17202.652187556832, 18474.28990623409, 16608.48790766844, 16698.873811438138]
\begin{figure*}[hbt!]
    \centering
    \captionsetup[subfigure]{justification=centering}
    \begin{subfigure}{0.45\textwidth}
        \centering
        \begin{tikzpicture}[trim axis left, trim axis right]
            % \centering
            \begin{axis}[
                % width=0.35 \textwidth,
                height = 6cm,
                % tick scale binop=\times,
        %             title={\ouralgo: FR vs. $\epsilon$},
        %             % ytick={0,0.5,1,1.5},
                xtick={0,50000,100000},
                xticklabel style = {font=\small,yshift=0.5ex},
                xlabel={Round ($t$)},
                scaled x ticks=false,
                xlabel near ticks,
                ytick={0,1000,2000,3000,4000,5000},
                ylabel={Fairness Regret},
                y tick label style={scaled ticks=base 10:-3},
                % ylabel near ticks,
                xlabel style={font=\small},
                ylabel style={font=\small},
                ymax=5000,
                ymajorgrids=true,
                xmajorgrids=true,
                grid style=dashed,
                legend style={at={(0,1)},anchor=north west,{column sep=0.05cm},nodes={scale=0.75, transform shape}},  
                legend entries={B0,\ourprivalgo,\algoname},
                ticklabel style={font=\small},
        %             % every axis plot/.append style={thick}
            ]
                    \addplot[color=mygreen,dashdotted,line width=0.75pt]
                    coordinates {(0, 0.18) (1000, 161.92) (2000, 287.2) (3000, 394.42) (4000, 491.36) (5000, 580.5) (6000, 660.96) (7000, 737.12) (8000, 812.65) (9000, 883.57) (10000, 951.72) (11000, 1016.9) (12000, 1079.92) (13000, 1140.13) (14000, 1199.22) (15000, 1255.4) (16000, 1308.86) (17000, 1362.7) (18000, 1414.52) (19000, 1465.87) (20000, 1516.07) (21000, 1564.61) (22000, 1613.12) (23000, 1660.64) (24000, 1706.62) (25000, 1752.63) (26000, 1797.61) (27000, 1841.73) (28000, 1884.7) (29000, 1926.11) (30000, 1966.93) (31000, 2007.33) (32000, 2047.71) (33000, 2088.16) (34000, 2128.41) (35000, 2167.74) (36000, 2206.17) (37000, 2244.04) (38000, 2281.72) (39000, 2319.6) (40000, 2356.52) (41000, 2393.18) (42000, 2429.21) (43000, 2465.67) (44000, 2501.41) (45000, 2536.29) (46000, 2570.86) (47000, 2604.63) (48000, 2637.75) (49000, 2670.07) (50000, 2701.97) (51000, 2733.48) (52000, 2765.45) (53000, 2797.36) (54000, 2828.54) (55000, 2859.64) (56000, 2889.53) (57000, 2919.57) (58000, 2948.84) (59000, 2977.98) (60000, 3006.47) (61000, 3034.28) (62000, 3061.83) (63000, 3089.75) (64000, 3117.21) (65000, 3144.55) (66000, 3171.16) (67000, 3197.81) (68000, 3224.76) (69000, 3251.53) (70000, 3277.92) (71000, 3304.03) (72000, 3330.33) (73000, 3356.38) (74000, 3382.2) (75000, 3407.64) (76000, 3432.79) (77000, 3457.82) (78000, 3482.94) (79000, 3507.86) (80000, 3532.68) (81000, 3557.34) (82000, 3581.66) (83000, 3605.59) (84000, 3629.98) (85000, 3654.56) (86000, 3679.36) (87000, 3704.15) (88000, 3728.29) (89000, 3752.44) (90000, 3776.84) (91000, 3801.22) (92000, 3825.48) (93000, 3849.51) (94000, 3873.34) (95000, 3896.78) (96000, 3919.81) (97000, 3942.65) (98000, 3965.88) (99000, 3988.82)};
                    
                    \addplot[color=mytealtres,solid,line width=0.5pt]
                    coordinates {(0, 0.14) (1000, 165.52) (2000, 239.66) (3000, 298.15) (4000, 354.05) (5000, 409.18) (6000, 463.17) (7000, 501.16) (8000, 529.6) (9000, 558.35) (10000, 586.95) (11000, 613.82) (12000, 638.13) (13000, 662.16) (14000, 686.81) (15000, 711.79) (16000, 738.44) (17000, 765.28) (18000, 791.62) (19000, 818.41) (20000, 842.2) (21000, 864.52) (22000, 886.59) (23000, 908.71) (24000, 931.19) (25000, 953.38) (26000, 975.55) (27000, 997.68) (28000, 1019.89) (29000, 1038.22) (30000, 1056.81) (31000, 1075.21) (32000, 1093.5) (33000, 1111.42) (34000, 1129.07) (35000, 1146.63) (36000, 1163.94) (37000, 1182.07) (38000, 1201.55) (39000, 1221.09) (40000, 1240.66) (41000, 1260.47) (42000, 1282.29) (43000, 1304.26) (44000, 1325.98) (45000, 1347.56) (46000, 1367.01) (47000, 1385.8) (48000, 1404.51) (49000, 1423.01) (50000, 1441.33) (51000, 1458.82) (52000, 1476.22) (53000, 1493.45) (54000, 1510.69) (55000, 1527.62) (56000, 1544.53) (57000, 1561.51) (58000, 1578.09) (59000, 1594.47) (60000, 1610.88) (61000, 1626.75) (62000, 1642.76) (63000, 1658.15) (64000, 1672.57) (65000, 1686.99) (66000, 1701.48) (67000, 1715.93) (68000, 1730.73) (69000, 1745.67) (70000, 1760.13) (71000, 1774.72) (72000, 1788.96) (73000, 1803.11) (74000, 1817.1) (75000, 1830.83) (76000, 1844.02) (77000, 1856.1) (78000, 1868.16) (79000, 1880.23) (80000, 1892.42) (81000, 1903.85) (82000, 1915.13) (83000, 1926.65) (84000, 1938.15) (85000, 1950.26) (86000, 1962.43) (87000, 1974.63) (88000, 1986.66) (89000, 1998.9) (90000, 2011.4) (91000, 2023.95) (92000, 2036.47) (93000, 2049.07) (94000, 2062.5) (95000, 2076.18) (96000, 2089.73) (97000, 2103.03) (98000, 2116.19) (99000, 2129.16)};
                    \addplot[color=myred,dashed,line width=0.75pt]
                    coordinates {(0, 0.29) (1000, 155.23) (2000, 242.72) (3000, 323.98) (4000, 393.21) (5000, 456.59) (6000, 519.25) (7000, 578.54) (8000, 637.22) (9000, 694.84) (10000, 751.24) (11000, 800.0) (12000, 831.32) (13000, 862.81) (14000, 894.79) (15000, 925.86) (16000, 950.37) (17000, 975.78) (18000, 1000.05) (19000, 1024.42) (20000, 1050.67) (21000, 1077.98) (22000, 1104.96) (23000, 1130.81) (24000, 1152.16) (25000, 1164.36) (26000, 1176.13) (27000, 1188.04) (28000, 1200.49) (29000, 1219.52) (30000, 1238.5) (31000, 1258.6) (32000, 1278.25) (33000, 1300.5) (34000, 1323.82) (35000, 1346.79) (36000, 1370.37) (37000, 1392.26) (38000, 1410.62) (39000, 1428.67) (40000, 1447.36) (41000, 1465.82) (42000, 1478.41) (43000, 1491.2) (44000, 1503.81) (45000, 1516.31) (46000, 1525.74) (47000, 1533.66) (48000, 1541.2) (49000, 1548.76) (50000, 1555.27) (51000, 1560.38) (52000, 1565.53) (53000, 1570.86) (54000, 1576.25) (55000, 1584.27) (56000, 1592.47) (57000, 1600.25) (58000, 1607.95) (59000, 1613.87) (60000, 1619.05) (61000, 1624.32) (62000, 1629.57) (63000, 1634.33) (64000, 1638.22) (65000, 1642.09) (66000, 1646.18) (67000, 1650.77) (68000, 1658.69) (69000, 1666.68) (70000, 1674.64) (71000, 1682.74) (72000, 1691.16) (73000, 1699.77) (74000, 1708.57) (75000, 1717.47) (76000, 1727.04) (77000, 1737.76) (78000, 1748.59) (79000, 1759.39) (80000, 1769.83) (81000, 1776.92) (82000, 1784.04) (83000, 1790.99) (84000, 1797.96) (85000, 1804.95) (86000, 1811.9) (87000, 1818.98) (88000, 1826.15) (89000, 1833.28) (90000, 1840.51) (91000, 1847.73) (92000, 1854.9) (93000, 1861.74) (94000, 1866.82) (95000, 1872.0) (96000, 1877.26) (97000, 1882.61) (98000, 1889.73) (99000, 1897.27)};
            \end{axis}
        \end{tikzpicture}
        \label{fig:rounds}
        \caption{ \\ }
    \end{subfigure}
    % Regret vs Rounds - Communication Protocols
    \begin{subfigure}{0.45\textwidth}
        \centering
        \begin{tikzpicture}[trim axis left, trim axis right]
            \begin{axis}[
                height = 6cm,
                % tick scale binop=\times,
        %             title={\ouralgo: FR vs. $\epsilon$},
        %             % ytick={0,0.5,1,1.5},
                xtick={0,50000,100000},
                xticklabel style = {font=\small,yshift=0.5ex},
                xlabel={Round ($t$)},
                scaled x ticks=false,
                xlabel near ticks,
                ytick={0,1000,2000,3000,4000,5000},
                y tick label style={scaled ticks=base 10:-3},
                ylabel={Fairness Regret},
                ylabel near ticks,
                xlabel style={font=\small},
                ylabel style={font=\small},
                ymax=5000,
                ymajorgrids=true,
                xmajorgrids=true,
                grid style=dashed,
                legend style={at={(0,1)},anchor=north west,{column sep=0.05cm},nodes={scale=0.75, transform shape}},  
                legend entries={\ourprivalgo, B1, B2},
                ticklabel style={font=\small},
        %             % every axis plot/.append style={thick}
            ]
                    \addplot[color=mytealtres,solid,line width=0.5pt]
                    coordinates {(0, 0.14) (1000, 165.52) (2000, 239.66) (3000, 298.15) (4000, 354.05) (5000, 409.18) (6000, 463.17) (7000, 501.16) (8000, 529.6) (9000, 558.35) (10000, 586.95) (11000, 613.82) (12000, 638.13) (13000, 662.16) (14000, 686.81) (15000, 711.79) (16000, 738.44) (17000, 765.28) (18000, 791.62) (19000, 818.41) (20000, 842.2) (21000, 864.52) (22000, 886.59) (23000, 908.71) (24000, 931.19) (25000, 953.38) (26000, 975.55) (27000, 997.68) (28000, 1019.89) (29000, 1038.22) (30000, 1056.81) (31000, 1075.21) (32000, 1093.5) (33000, 1111.42) (34000, 1129.07) (35000, 1146.63) (36000, 1163.94) (37000, 1182.07) (38000, 1201.55) (39000, 1221.09) (40000, 1240.66) (41000, 1260.47) (42000, 1282.29) (43000, 1304.26) (44000, 1325.98) (45000, 1347.56) (46000, 1367.01) (47000, 1385.8) (48000, 1404.51) (49000, 1423.01) (50000, 1441.33) (51000, 1458.82) (52000, 1476.22) (53000, 1493.45) (54000, 1510.69) (55000, 1527.62) (56000, 1544.53) (57000, 1561.51) (58000, 1578.09) (59000, 1594.47) (60000, 1610.88) (61000, 1626.75) (62000, 1642.76) (63000, 1658.15) (64000, 1672.57) (65000, 1686.99) (66000, 1701.48) (67000, 1715.93) (68000, 1730.73) (69000, 1745.67) (70000, 1760.13) (71000, 1774.72) (72000, 1788.96) (73000, 1803.11) (74000, 1817.1) (75000, 1830.83) (76000, 1844.02) (77000, 1856.1) (78000, 1868.16) (79000, 1880.23) (80000, 1892.42) (81000, 1903.85) (82000, 1915.13) (83000, 1926.65) (84000, 1938.15) (85000, 1950.26) (86000, 1962.43) (87000, 1974.63) (88000, 1986.66) (89000, 1998.9) (90000, 2011.4) (91000, 2023.95) (92000, 2036.47) (93000, 2049.07) (94000, 2062.5) (95000, 2076.18) (96000, 2089.73) (97000, 2103.03) (98000, 2116.19) (99000, 2129.16)};
                    \addplot[color=teal,dashdotted,line width=0.75pt]
                    coordinates {(0, 0.12) (1000, 175.92) (2000, 305.93) (3000, 434.43) (4000, 560.78) (5000, 668.58) (6000, 747.35) (7000, 793.72) (8000, 841.06) (9000, 888.09) (10000, 933.67) (11000, 977.13) (12000, 1019.41) (13000, 1060.8) (14000, 1101.6) (15000, 1141.97) (16000, 1181.67) (17000, 1221.17) (18000, 1259.49) (19000, 1296.86) (20000, 1333.48) (21000, 1369.27) (22000, 1406.29) (23000, 1442.27) (24000, 1478.35) (25000, 1514.06) (26000, 1550.0) (27000, 1586.19) (28000, 1621.93) (29000, 1657.82) (30000, 1693.15) (31000, 1726.54) (32000, 1760.68) (33000, 1794.42) (34000, 1827.82) (35000, 1860.28) (36000, 1893.32) (37000, 1925.74) (38000, 1958.47) (39000, 1990.85) (40000, 2023.0) (41000, 2054.65) (42000, 2086.31) (43000, 2117.41) (44000, 2148.68) (45000, 2179.24) (46000, 2209.83) (47000, 2239.89) (48000, 2269.64) (49000, 2299.43) (50000, 2328.82) (51000, 2358.27) (52000, 2387.84) (53000, 2417.68) (54000, 2446.68) (55000, 2475.92) (56000, 2505.08) (57000, 2533.42) (58000, 2562.07) (59000, 2590.61) (60000, 2619.72) (61000, 2647.66) (62000, 2675.25) (63000, 2697.91) (64000, 2707.03) (65000, 2716.08) (66000, 2725.17) (67000, 2734.24) (68000, 2743.41) (69000, 2752.49) (70000, 2761.42) (71000, 2770.21) (72000, 2779.19) (73000, 2788.21) (74000, 2797.17) (75000, 2806.05) (76000, 2814.84) (77000, 2823.66) (78000, 2832.43) (79000, 2841.27) (80000, 2850.02) (81000, 2858.75) (82000, 2867.55) (83000, 2876.39) (84000, 2885.17) (85000, 2893.79) (86000, 2902.31) (87000, 2911.02) (88000, 2919.61) (89000, 2928.24) (90000, 2936.85) (91000, 2945.47) (92000, 2954.22) (93000, 2962.95) (94000, 2971.51) (95000, 2980.03) (96000, 2988.69) (97000, 2997.11) (98000, 3005.57) (99000, 3014.11)};
                    \addplot[color=orange,dashed,line width=0.75pt]
                    coordinates {(0, 0.15) (1000, 149.7) (2000, 197.03) (3000, 239.52) (4000, 292.27) (5000, 346.18) (6000, 398.76) (7000, 448.92) (8000, 496.73) (9000, 543.58) (10000, 590.41) (11000, 636.87) (12000, 682.76) (13000, 722.25) (14000, 741.64) (15000, 761.12) (16000, 780.5) (17000, 800.46) (18000, 820.57) (19000, 840.26) (20000, 859.98) (21000, 879.85) (22000, 900.22) (23000, 920.4) (24000, 940.03) (25000, 959.89) (26000, 977.25) (27000, 992.42) (28000, 1007.59) (29000, 1022.61) (30000, 1037.92) (31000, 1052.96) (32000, 1067.9) (33000, 1082.89) (34000, 1097.81) (35000, 1113.0) (36000, 1128.17) (37000, 1143.13) (38000, 1158.61) (39000, 1173.74) (40000, 1188.83) (41000, 1204.13) (42000, 1219.27) (43000, 1234.56) (44000, 1250.06) (45000, 1265.02) (46000, 1280.08) (47000, 1294.93) (48000, 1309.58) (49000, 1324.11) (50000, 1338.63) (51000, 1353.04) (52000, 1365.15) (53000, 1376.66) (54000, 1388.45) (55000, 1400.34) (56000, 1411.98) (57000, 1423.81) (58000, 1435.53) (59000, 1447.43) (60000, 1459.14) (61000, 1470.79) (62000, 1482.26) (63000, 1493.81) (64000, 1505.33) (65000, 1517.11) (66000, 1528.66) (67000, 1540.09) (68000, 1551.75) (69000, 1563.1) (70000, 1574.44) (71000, 1586.04) (72000, 1597.57) (73000, 1608.94) (74000, 1620.18) (75000, 1631.41) (76000, 1642.74) (77000, 1654.0) (78000, 1665.37) (79000, 1676.61) (80000, 1688.19) (81000, 1699.5) (82000, 1710.9) (83000, 1722.37) (84000, 1733.79) (85000, 1745.02) (86000, 1756.24) (87000, 1767.57) (88000, 1778.69) (89000, 1790.03) (90000, 1801.19) (91000, 1812.45) (92000, 1823.75) (93000, 1834.88) (94000, 1846.12) (95000, 1857.7) (96000, 1868.9) (97000, 1880.35) (98000, 1891.74) (99000, 1903.45)};
            \end{axis}
        \end{tikzpicture}
        \caption{ \\ }
        \label{fig:comm}
    \end{subfigure}
    
    \begin{subfigure}{0.45\textwidth}
        \centering
        \begin{tikzpicture}[trim axis left, trim axis right]
        \begin{axis}[
            height = 6cm,
    %             tick scale binop=\times,
    %             title={\ouralgo: FR vs. $\epsilon$},
    %             % ytick={0,0.5,1,1.5},
            % xtick={10,20,30,40},
            symbolic x coords={10, 20, 30, 40},
            xlabel={Agents ($m$)},
            xlabel near ticks,
            y tick label style={scaled ticks=base 10:-3},            
            ylabel={Fairness Regret},
            ylabel near ticks,
            ylabel style={font=\small},
            xlabel style={font=\small},
            ymajorgrids=true,
            xmajorgrids=true,
            grid style=dashed,
            legend style={at={(1,1)},{column sep=0.01cm},nodes={scale=0.75, transform shape}},  
            legend entries={\ourprivalgo,\algoname},
            ticklabel style={font=\small},
    %             % every axis plot/.append style={thick}
        ]       
        
                \addplot[color=mytealtres,loosely dashed, every mark/.append style={solid, fill=gray},mark=square*,line width=0.75pt]
                coordinates {
                (10, 2129.16)
                (20, 1467.50)
                (30, 1051.314)
                (40, 1011.85)}
               ;
                
                \addplot[color=myred,loosely dashed, every mark/.append style={solid, fill=gray},mark=square*,line width=0.75pt]
                coordinates {
                (10, 1142.79)
                (20, 1089.57)
                (30, 925.21)
                (40, 777.14)};

                % \addplot[color=mytealdos,loosely dashed, every mark/.append style={solid, fill=gray},mark=star]    
                % coordinates {
                % (2, 2341.131869558899)
                % (4, 2182.9453906995764)
                % (6, 1426.6962503216969)
                % (8, 1410.2518116239437)
                % (10, 1296.2869663935865)};

                % \addplot[color=green]    
                % coordinates {
                % (2, 3152)
                % (10, 3152)};
            \end{axis}
        \end{tikzpicture}
        \caption{}
        \label{fig:agents}
    \end{subfigure}
    % Total Regret vs Epsilon
    \begin{subfigure}{0.45\textwidth}
        \centering
        \begin{tikzpicture}[trim axis left, trim axis right]
            \begin{axis}[
                height = 6cm,
                % tick scale binop=\times,
%             title={\ouralgo: FR vs. $\epsilon$},
%             % ytick={0,0.5,1,1.5},
                xtick={0.1,1,10},
                xlabel={Privacy Budget (\(\epsilon\))},
                xlabel near ticks,
                xlabel style={font=\small},
                xmode=log,
                log ticks with fixed point,
                ylabel={Fairness Regret},
                ytick={0,10000,20000},
                ylabel near ticks,
                ylabel style={font=\small},
                % ylabel near ticks,
                % yticklabel pos=left, % the '*' avoids arrow heads
                ymajorgrids=true,
                xmajorgrids=true,
                grid style=dashed,
                legend style={at={(1,1)},{column sep=0.01cm},nodes={scale=0.75, transform shape}},  
                legend entries={\ourprivalgo,B0},
                ticklabel style={font=\small},
    %             % every axis plot/.append style={thick}
            ]
                \addplot[color=mytealtres,loosely dashed,every mark/.append style={solid, fill=gray},mark=square*,line width=0.75pt]
                coordinates {
                (0.1, 19683.06)
                (1, 2176.16)
                (10, 845.98)};
                \addplot[color=mygreen,line width=0.5pt]    
                coordinates {
                (0.1, 3988)
                (1, 3988)
                (10, 3988)};
                
            \end{axis}
        \end{tikzpicture}
    \caption{}
    \label{fig:epsilon}
    \end{subfigure}
    \caption{\textbf{(a)} Exp 1 : Fairness Regret vs. Rounds for single-agent baseline and proposed federated learning algorithms (m=10)  \textbf{(b)} Exp 2 : Fairness Regret vs. Rounds for different communication protocol baselines and proposed algorithms (m=10)  \textbf{(c)} Exp 3 : Fairness Regret trend w.r.t. number of agents (t=100,000)  \textbf{(d)} Exp 4 : Fairness Regret trend w.r.t. privacy budget (t=100,000) }
    \label{fig:epsilon}
\end{figure*}
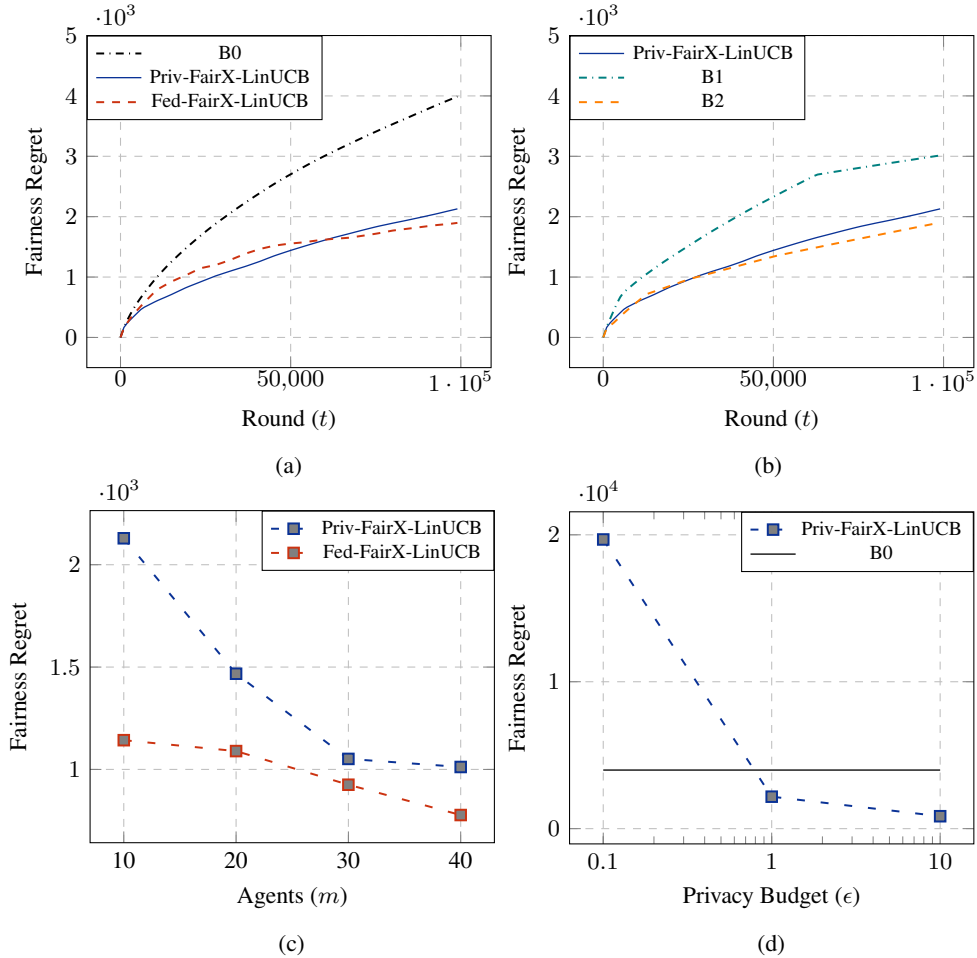

%% file: Plots/Round-Appendix.tex
% [3009.537541158756, 2022.1482217010616, 2161.8114335952987, 1383.9164497285542, 1632.1401226366554, 1224.2291907517565]

% [3027.814251030743, 3053.0032979190432, 2968.8011925361707, 2288.8568274001022, 2260.2942188295847, 1750.26074339728]

% [3145.1049759108437, 17766.264776730597, 17202.652187556832, 18474.28990623409, 16608.48790766844, 16698.873811438138]
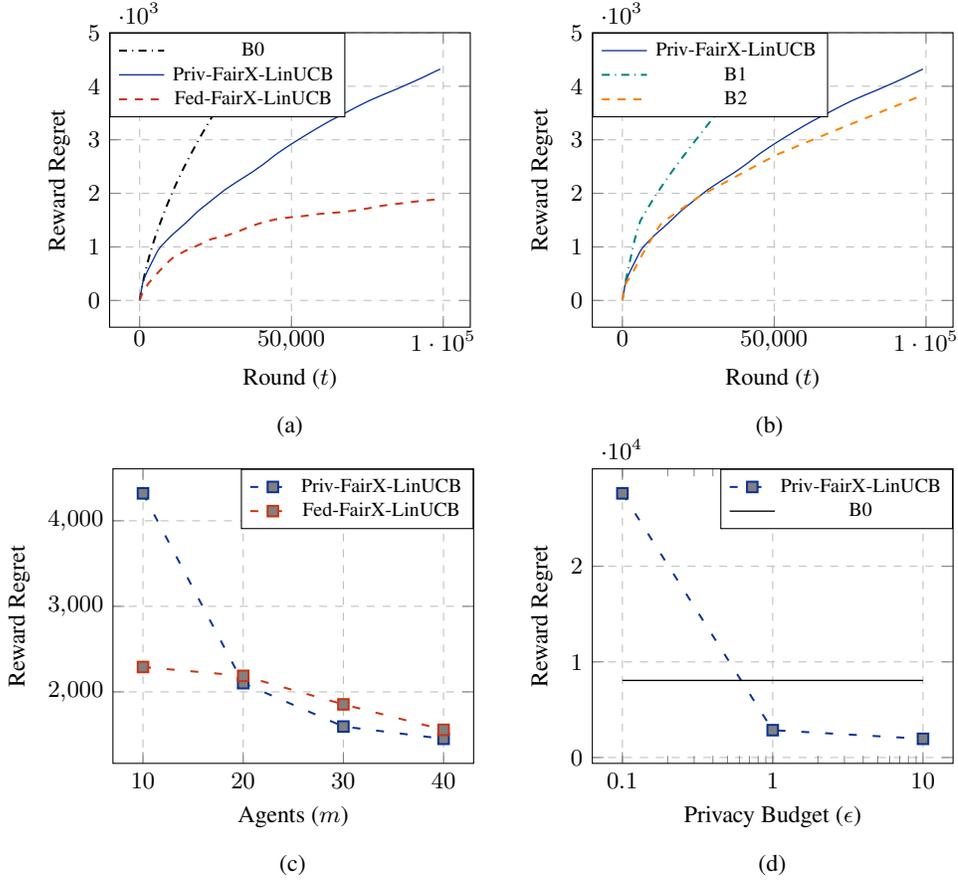
\begin{figure*}[hbt!]
    \centering
    \captionsetup[subfigure]{justification=centering}
    \begin{subfigure}{0.45\textwidth}
        \centering
        \begin{tikzpicture}[trim axis left, trim axis right]
            % \centering
            \begin{axis}[
                % width=0.35 \textwidth,
                height = 5.5cm,
                % tick scale binop=\times,
        %             title={\ouralgo: FR vs. $\epsilon$},
        %             % ytick={0,0.5,1,1.5},
                xtick={0,50000,100000},
                xticklabel style = {font=\small,yshift=0.5ex},
                xlabel={Round ($t$)},
                scaled x ticks=false,
                xlabel near ticks,
                ytick={0,1000,2000,3000,4000,5000},
                ylabel={Reward Regret},
                y tick label style={scaled ticks=base 10:-3},
                % ylabel near ticks,
                xlabel style={font=\small},
                ylabel style={font=\small},
                ymax=5000,
                ymajorgrids=true,
                xmajorgrids=true,
                grid style=dashed,
                legend style={at={(0,1)},anchor=north west,{column sep=0.05cm},nodes={scale=0.75, transform shape}},  
                legend entries={B0,\ourprivalgo,\algoname},
                ticklabel style={font=\small},
        %             % every axis plot/.append style={thick}
            ]
                    \addplot[color=mygreen,dashdotted,line width=0.75pt]
                    coordinates {(0, 0.36) (1000, 322.27) (2000, 573.75) (3000, 789.76) (4000, 985.08) (5000, 1164.94) (6000, 1327.41) (7000, 1481.19) (8000, 1633.82) (9000, 1777.28) (10000, 1915.34) (11000, 2047.19) (12000, 2174.63) (13000, 2296.41) (14000, 2416.01) (15000, 2529.81) (16000, 2637.98) (17000, 2746.86) (18000, 2851.52) (19000, 2955.33) (20000, 3056.92) (21000, 3155.14) (22000, 3253.28) (23000, 3349.54) (24000, 3442.61) (25000, 3535.64) (26000, 3626.67) (27000, 3715.97) (28000, 3802.94) (29000, 3886.78) (30000, 3969.47) (31000, 4051.27) (32000, 4133.02) (33000, 4214.91) (34000, 4296.4) (35000, 4376.01) (36000, 4453.86) (37000, 4530.52) (38000, 4606.85) (39000, 4683.63) (40000, 4758.37) (41000, 4832.69) (42000, 4905.62) (43000, 4979.39) (44000, 5051.72) (45000, 5122.42) (46000, 5192.36) (47000, 5260.72) (48000, 5327.8) (49000, 5393.27) (50000, 5457.85) (51000, 5521.6) (52000, 5586.3) (53000, 5650.89) (54000, 5714.02) (55000, 5776.97) (56000, 5837.52) (57000, 5898.38) (58000, 5957.64) (59000, 6016.62) (60000, 6074.26) (61000, 6130.58) (62000, 6186.4) (63000, 6242.86) (64000, 6298.45) (65000, 6353.79) (66000, 6407.67) (67000, 6461.57) (68000, 6516.14) (69000, 6570.31) (70000, 6623.72) (71000, 6676.6) (72000, 6729.87) (73000, 6782.58) (74000, 6834.83) (75000, 6886.32) (76000, 6937.22) (77000, 6987.86) (78000, 7038.68) (79000, 7089.1) (80000, 7139.38) (81000, 7189.29) (82000, 7238.54) (83000, 7287.04) (84000, 7336.45) (85000, 7386.22) (86000, 7436.44) (87000, 7486.65) (88000, 7535.52) (89000, 7584.46) (90000, 7633.81) (91000, 7683.23) (92000, 7732.36) (93000, 7781.02) (94000, 7829.2) (95000, 7876.68) (96000, 7923.32) (97000, 7969.63) (98000, 8016.69) (99000, 8063.13)};
                    \addplot[color=mytealtres,solid,line width=0.5pt]
                    coordinates {(0, 0.28) (1000, 327.63) (2000, 477.97) (3000, 597.26) (4000, 710.66) (5000, 822.57) (6000, 932.01) (7000, 1009.25) (8000, 1067.22) (9000, 1125.8) (10000, 1184.21) (11000, 1239.0) (12000, 1288.6) (13000, 1337.56) (14000, 1387.84) (15000, 1438.67) (16000, 1492.71) (17000, 1547.03) (18000, 1600.42) (19000, 1654.72) (20000, 1702.94) (21000, 1748.27) (22000, 1793.01) (23000, 1838.01) (24000, 1883.65) (25000, 1928.62) (26000, 1973.51) (27000, 2018.52) (28000, 2063.62) (29000, 2100.98) (30000, 2138.86) (31000, 2176.31) (32000, 2213.61) (33000, 2250.14) (34000, 2286.1) (35000, 2321.81) (36000, 2357.08) (37000, 2394.01) (38000, 2433.65) (39000, 2473.5) (40000, 2513.29) (41000, 2553.71) (42000, 2598.07) (43000, 2642.77) (44000, 2687.0) (45000, 2730.98) (46000, 2770.47) (47000, 2808.7) (48000, 2846.73) (49000, 2884.42) (50000, 2921.69) (51000, 2957.3) (52000, 2992.66) (53000, 3027.74) (54000, 3062.78) (55000, 3097.25) (56000, 3131.67) (57000, 3166.24) (58000, 3199.93) (59000, 3233.25) (60000, 3266.56) (61000, 3298.82) (62000, 3331.32) (63000, 3362.55) (64000, 3391.85) (65000, 3421.16) (66000, 3450.64) (67000, 3480.03) (68000, 3510.14) (69000, 3540.51) (70000, 3569.86) (71000, 3599.61) (72000, 3628.51) (73000, 3657.32) (74000, 3685.68) (75000, 3713.61) (76000, 3740.38) (77000, 3764.96) (78000, 3789.48) (79000, 3814.06) (80000, 3838.86) (81000, 3862.12) (82000, 3885.07) (83000, 3908.54) (84000, 3931.95) (85000, 3956.6) (86000, 3981.39) (87000, 4006.22) (88000, 4030.69) (89000, 4055.62) (90000, 4081.04) (91000, 4106.6) (92000, 4132.02) (93000, 4157.61) (94000, 4184.9) (95000, 4212.69) (96000, 4240.15) (97000, 4267.28) (98000, 4294.03) (99000, 4320.39)};
                    \addplot[color=myred,dashed,line width=0.75pt]
                    coordinates {(0, 0.29) (1000, 155.23) (2000, 242.72) (3000, 323.98) (4000, 393.21) (5000, 456.59) (6000, 519.25) (7000, 578.54) (8000, 637.22) (9000, 694.84) (10000, 751.24) (11000, 800.0) (12000, 831.32) (13000, 862.81) (14000, 894.79) (15000, 925.86) (16000, 950.37) (17000, 975.78) (18000, 1000.05) (19000, 1024.42) (20000, 1050.67) (21000, 1077.98) (22000, 1104.96) (23000, 1130.81) (24000, 1152.16) (25000, 1164.36) (26000, 1176.13) (27000, 1188.04) (28000, 1200.49) (29000, 1219.52) (30000, 1238.5) (31000, 1258.6) (32000, 1278.25) (33000, 1300.5) (34000, 1323.82) (35000, 1346.79) (36000, 1370.37) (37000, 1392.26) (38000, 1410.62) (39000, 1428.67) (40000, 1447.36) (41000, 1465.82) (42000, 1478.41) (43000, 1491.2) (44000, 1503.81) (45000, 1516.31) (46000, 1525.74) (47000, 1533.66) (48000, 1541.2) (49000, 1548.76) (50000, 1555.27) (51000, 1560.38) (52000, 1565.53) (53000, 1570.86) (54000, 1576.25) (55000, 1584.27) (56000, 1592.47) (57000, 1600.25) (58000, 1607.95) (59000, 1613.87) (60000, 1619.05) (61000, 1624.32) (62000, 1629.57) (63000, 1634.33) (64000, 1638.22) (65000, 1642.09) (66000, 1646.18) (67000, 1650.77) (68000, 1658.69) (69000, 1666.68) (70000, 1674.64) (71000, 1682.74) (72000, 1691.16) (73000, 1699.77) (74000, 1708.57) (75000, 1717.47) (76000, 1727.04) (77000, 1737.76) (78000, 1748.59) (79000, 1759.39) (80000, 1769.83) (81000, 1776.92) (82000, 1784.04) (83000, 1790.99) (84000, 1797.96) (85000, 1804.95) (86000, 1811.9) (87000, 1818.98) (88000, 1826.15) (89000, 1833.28) (90000, 1840.51) (91000, 1847.73) (92000, 1854.9) (93000, 1861.74) (94000, 1866.82) (95000, 1872.0) (96000, 1877.26) (97000, 1882.61) (98000, 1889.73) (99000, 1897.27)};
            \end{axis}
        \end{tikzpicture}
        \label{fig:rounds}
        \caption{ \\ }
    \end{subfigure}
    % Regret vs Rounds - Communication Protocols
    \begin{subfigure}{0.45\textwidth}
        \centering
        \begin{tikzpicture}[trim axis left, trim axis right]
            \begin{axis}[
                height = 5.5cm,
                % tick scale binop=\times,
        %             title={\ouralgo: FR vs. $\epsilon$},
        %             % ytick={0,0.5,1,1.5},
                xtick={0,50000,100000},
                xticklabel style = {font=\small,yshift=0.5ex},
                xlabel={Round ($t$)},
                scaled x ticks=false,
                xlabel near ticks,
                ytick={0,1000,2000,3000,4000,5000},
                y tick label style={scaled ticks=base 10:-3},
                ylabel={Reward Regret},
                ylabel near ticks,
                xlabel style={font=\small},
                ylabel style={font=\small},
                ymax=5000,
                ymajorgrids=true,
                xmajorgrids=true,
                grid style=dashed,
                legend style={at={(0,1)},anchor=north west,{column sep=0.05cm},nodes={scale=0.75, transform shape}},  
                legend entries={\ourprivalgo, B1, B2},
                ticklabel style={font=\small},
        %             % every axis plot/.append style={thick}
            ]
                    \addplot[color=mytealtres,solid,line width=0.5pt]
                    coordinates {(0, 0.28) (1000, 327.63) (2000, 477.97) (3000, 597.26) (4000, 710.66) (5000, 822.57) (6000, 932.01) (7000, 1009.25) (8000, 1067.22) (9000, 1125.8) (10000, 1184.21) (11000, 1239.0) (12000, 1288.6) (13000, 1337.56) (14000, 1387.84) (15000, 1438.67) (16000, 1492.71) (17000, 1547.03) (18000, 1600.42) (19000, 1654.72) (20000, 1702.94) (21000, 1748.27) (22000, 1793.01) (23000, 1838.01) (24000, 1883.65) (25000, 1928.62) (26000, 1973.51) (27000, 2018.52) (28000, 2063.62) (29000, 2100.98) (30000, 2138.86) (31000, 2176.31) (32000, 2213.61) (33000, 2250.14) (34000, 2286.1) (35000, 2321.81) (36000, 2357.08) (37000, 2394.01) (38000, 2433.65) (39000, 2473.5) (40000, 2513.29) (41000, 2553.71) (42000, 2598.07) (43000, 2642.77) (44000, 2687.0) (45000, 2730.98) (46000, 2770.47) (47000, 2808.7) (48000, 2846.73) (49000, 2884.42) (50000, 2921.69) (51000, 2957.3) (52000, 2992.66) (53000, 3027.74) (54000, 3062.78) (55000, 3097.25) (56000, 3131.67) (57000, 3166.24) (58000, 3199.93) (59000, 3233.25) (60000, 3266.56) (61000, 3298.82) (62000, 3331.32) (63000, 3362.55) (64000, 3391.85) (65000, 3421.16) (66000, 3450.64) (67000, 3480.03) (68000, 3510.14) (69000, 3540.51) (70000, 3569.86) (71000, 3599.61) (72000, 3628.51) (73000, 3657.32) (74000, 3685.68) (75000, 3713.61) (76000, 3740.38) (77000, 3764.96) (78000, 3789.48) (79000, 3814.06) (80000, 3838.86) (81000, 3862.12) (82000, 3885.07) (83000, 3908.54) (84000, 3931.95) (85000, 3956.6) (86000, 3981.39) (87000, 4006.22) (88000, 4030.69) (89000, 4055.62) (90000, 4081.04) (91000, 4106.6) (92000, 4132.02) (93000, 4157.61) (94000, 4184.9) (95000, 4212.69) (96000, 4240.15) (97000, 4267.28) (98000, 4294.03) (99000, 4320.39)};
                    \addplot[color=teal,dashdotted,line width=0.75pt]
                    coordinates {(0, 0.24) (1000, 348.41) (2000, 609.31) (3000, 867.75) (4000, 1121.94) (5000, 1339.13) (6000, 1498.36) (7000, 1592.52) (8000, 1688.66) (9000, 1784.25) (10000, 1876.94) (11000, 1965.15) (12000, 2050.81) (13000, 2135.15) (14000, 2217.96) (15000, 2299.92) (16000, 2380.4) (17000, 2460.56) (18000, 2538.4) (19000, 2614.2) (20000, 2688.38) (21000, 2761.01) (22000, 2836.2) (23000, 2909.23) (24000, 2982.39) (25000, 3054.62) (26000, 3127.36) (27000, 3200.78) (28000, 3273.1) (29000, 3345.86) (30000, 3417.49) (31000, 3485.07) (32000, 3554.31) (33000, 3622.74) (34000, 3690.42) (35000, 3756.06) (36000, 3823.13) (37000, 3888.89) (38000, 3955.3) (39000, 4020.87) (40000, 4085.94) (41000, 4150.15) (42000, 4214.33) (43000, 4277.2) (44000, 4340.55) (45000, 4402.7) (46000, 4464.54) (47000, 4525.42) (48000, 4585.62) (49000, 4646.04) (50000, 4705.43) (51000, 4765.06) (52000, 4824.92) (53000, 4885.2) (54000, 4943.93) (55000, 5003.01) (56000, 5062.02) (57000, 5119.39) (58000, 5177.35) (59000, 5235.15) (60000, 5293.99) (61000, 5350.58) (62000, 5406.56) (63000, 5452.31) (64000, 5470.62) (65000, 5488.8) (66000, 5507.03) (67000, 5525.27) (68000, 5543.69) (69000, 5561.9) (70000, 5579.84) (71000, 5597.48) (72000, 5615.52) (73000, 5633.65) (74000, 5651.62) (75000, 5669.44) (76000, 5687.06) (77000, 5704.78) (78000, 5722.38) (79000, 5740.11) (80000, 5757.73) (81000, 5775.26) (82000, 5792.95) (83000, 5810.66) (84000, 5828.31) (85000, 5845.61) (86000, 5862.73) (87000, 5880.28) (88000, 5897.52) (89000, 5914.83) (90000, 5932.11) (91000, 5949.41) (92000, 5966.99) (93000, 5984.55) (94000, 6001.73) (95000, 6018.86) (96000, 6036.21) (97000, 6053.16) (98000, 6070.16) (99000, 6087.3)};
                    \addplot[color=orange,dashed,line width=0.75pt]
                    coordinates {(0, 0.3) (1000, 298.44) (2000, 394.53) (3000, 480.77) (4000, 586.91) (5000, 695.21) (6000, 801.04) (7000, 901.78) (8000, 998.05) (9000, 1092.49) (10000, 1187.15) (11000, 1280.57) (12000, 1373.14) (13000, 1452.7) (14000, 1491.74) (15000, 1530.99) (16000, 1570.08) (17000, 1610.17) (18000, 1650.62) (19000, 1690.29) (20000, 1730.07) (21000, 1770.06) (22000, 1811.08) (23000, 1851.79) (24000, 1891.3) (25000, 1931.26) (26000, 1966.26) (27000, 1996.87) (28000, 2027.49) (29000, 2057.8) (30000, 2088.72) (31000, 2119.1) (32000, 2149.2) (33000, 2179.43) (34000, 2209.47) (35000, 2240.08) (36000, 2270.61) (37000, 2300.79) (38000, 2332.02) (39000, 2362.58) (40000, 2393.02) (41000, 2423.9) (42000, 2454.46) (43000, 2485.32) (44000, 2516.61) (45000, 2546.88) (46000, 2577.14) (47000, 2607.09) (48000, 2636.59) (49000, 2665.92) (50000, 2695.21) (51000, 2724.22) (52000, 2748.67) (53000, 2771.96) (54000, 2795.77) (55000, 2819.84) (56000, 2843.34) (57000, 2867.3) (58000, 2890.98) (59000, 2915.0) (60000, 2938.65) (61000, 2962.22) (62000, 2985.38) (63000, 3008.7) (64000, 3032.0) (65000, 3055.84) (66000, 3079.2) (67000, 3102.34) (68000, 3125.93) (69000, 3148.91) (70000, 3171.89) (71000, 3195.33) (72000, 3218.59) (73000, 3241.6) (74000, 3264.29) (75000, 3287.0) (76000, 3309.93) (77000, 3332.75) (78000, 3355.69) (79000, 3378.44) (80000, 3401.88) (81000, 3424.82) (82000, 3447.85) (83000, 3470.99) (84000, 3494.13) (85000, 3516.77) (86000, 3539.46) (87000, 3562.36) (88000, 3584.83) (89000, 3607.78) (90000, 3630.26) (91000, 3653.04) (92000, 3675.9) (93000, 3698.4) (94000, 3721.08) (95000, 3744.47) (96000, 3767.11) (97000, 3790.27) (98000, 3813.34) (99000, 3837.0)};
            \end{axis}
        \end{tikzpicture}
        \caption{ \\ }
        \label{fig:comm}
    \end{subfigure}

    \begin{subfigure}{0.45\textwidth}
        \centering
        \begin{tikzpicture}[trim axis left, trim axis right]
        \begin{axis}[
            height = 5.5cm,
    %             tick scale binop=\times,
    %             title={\ouralgo: FR vs. $\epsilon$},
    %             % ytick={0,0.5,1,1.5},
            % xtick={10,20,30,40},
            % symbolic x coords={10, 20, 30, 40},
            xlabel={Agents ($m$)},
            xlabel near ticks,
            % y tick label style={scaled ticks=base 10:-3},            
            ylabel={Reward Regret},
            ylabel near ticks,
            ylabel style={font=\small},
            xlabel style={font=\small},
            ymajorgrids=true,
            xmajorgrids=true,
            grid style=dashed,
            legend style={at={(1,1)},{column sep=0.01cm},nodes={scale=0.75, transform shape}},  
            legend entries={\ourprivalgo,\algoname},
            ticklabel style={font=\small},
    %             % every axis plot/.append style={thick}
        ]       
        
                \addplot[color=mytealtres,loosely dashed, every mark/.append style={solid, fill=gray},mark=square*,line width=0.75pt]
                coordinates {
                (10, 4320.39)
                (20, 2101.95)
                (30, 1594.58)
                (40, 1451.2)}
               ;
                
                \addplot[color=myred,loosely dashed, every mark/.append style={solid, fill=gray},mark=square*,line width=0.75pt]
                coordinates {
                (10, 2290.67)
                (20, 2187.44)
                (30, 1853.84)
                (40, 1556.34)};

                % \addplot[color=mytealdos,loosely dashed, every mark/.append style={solid, fill=gray},mark=star]    
                % coordinates {
                % (2, 2341.131869558899)
                % (4, 2182.9453906995764)
                % (6, 1426.6962503216969)
                % (8, 1410.2518116239437)
                % (10, 1296.2869663935865)};

                % \addplot[color=green]    
                % coordinates {
                % (2, 3152)
                % (10, 3152)};
            \end{axis}
        \end{tikzpicture}
        \caption{}
        \label{fig:agents}
    \end{subfigure}
    % Total Regret vs Epsilon
    \begin{subfigure}{0.45\textwidth}
        \centering
        \begin{tikzpicture}[trim axis left, trim axis right]
            \begin{axis}[
                height = 5.5cm,
                % tick scale binop=\times,
%             title={\ouralgo: FR vs. $\epsilon$},
%             % ytick={0,0.5,1,1.5},
                xtick={0.1,1,10},
                xlabel={Privacy Budget (\(\epsilon\))},
                xlabel near ticks,
                xlabel style={font=\small},
                xmode=log,
                log ticks with fixed point,
                ylabel={Reward Regret},
                ytick={0,10000,20000},
                ylabel near ticks,
                ylabel style={font=\small},
                % ylabel near ticks,
                % yticklabel pos=left, % the '*' avoids arrow heads
                ymajorgrids=true,
                xmajorgrids=true,
                grid style=dashed,
                legend style={at={(1,1)},{column sep=0.01cm},nodes={scale=0.75, transform shape}},  
                legend entries={\ourprivalgo,B0},
                ticklabel style={font=\small},
    %             % every axis plot/.append style={thick}
            ]
                \addplot[color=mytealtres,loosely dashed,every mark/.append style={solid, fill=gray},mark=square*,line width=0.75pt]
                coordinates {
                (0.1, 27624.51)
                (1, 2861.13)
                (10, 1956.14)};
                \addplot[color=mygreen,line width=0.5pt]    
                coordinates {
                (0.1,  8063.13)
                (1,  8063.13)
                (10,  8063.13)};
                
            \end{axis}
        \end{tikzpicture}
    \caption{}
    \label{fig:epsilon}
    \end{subfigure}
    \caption{\textbf{(a)} Exp 1 : Reward Regret vs. Rounds for single-agent baseline and proposed federated learning algorithms (m=10)  \textbf{(b)} Exp 2 : Reward Regret vs. Rounds for different communication protocol baselines and proposed algorithms (m=10)  \textbf{(c)} Exp 3 : Reward Regret trend w.r.t. number of agents (t=100,000)  \textbf{(d)} Exp 4 : Reward Regret trend w.r.t. privacy budget (t=100,000) }
    \label{fig:epsilon}
\end{figure*}